\pdfoutput=1
\documentclass[format=acmsmall, review=false, screen=true]{acmart}

\usepackage{booktabs} 

\usepackage{framed}
\usepackage{balance}
\usepackage{epsfig}
\usepackage{color}
\usepackage{subfigure}
\usepackage{multirow,tabularx}
\usepackage{amssymb}
\usepackage{mathtools}
\usepackage{graphicx}
\usepackage{multirow}
\usepackage{bm}
\usepackage{csquotes}
\usepackage{array}
\usepackage[yyyymmdd,hhmmss]{datetime}
\usepackage[textsize=footnotesize,bordercolor=black!20]{todonotes}
\usepackage{xcolor,colortbl}
\usepackage{amssymb}
\usepackage{pifont}
\usepackage[algo2e,titlenumbered,ruled]{algorithm2e} 
\usepackage{lipsum,environ,amsmath}
\usepackage{slashbox,booktabs,amsmath}

\usepackage{listings}
\definecolor{codegreen}{rgb}{0,0.6,0}
\definecolor{codegray}{rgb}{0.5,0.5,0.5}
\definecolor{codepurple}{rgb}{0.58,0,0.82}
\definecolor{backcolour}{rgb}{0.95,0.95,0.92}

\lstdefinestyle{mystyle}{
    backgroundcolor=\color{backcolour},   
    commentstyle=\color{codegreen},
    keywordstyle=\color{magenta},
    numberstyle=\tiny\color{codegray},
    stringstyle=\color{codepurple},
    basicstyle=\ttfamily\bfseries\footnotesize,
    breakatwhitespace=false,         
    breaklines=true,                 
    captionpos=b,                    
    keepspaces=true,                 
    numbers=left,                    
    numbersep=5pt,                  
    showspaces=false,                
    showstringspaces=false,
    showtabs=false,                  
    tabsize=2
}

\usepackage{xspace}

\usepackage{bbm}

\newcounter{Lcount}
\newcommand{\numsquishlist}{
   \begin{list}{\arabic{Lcount}. }
    { \usecounter{Lcount}
 \setlength{\itemsep}{-.1ex}      \setlength{\parsep}{0ex}
      \setlength{\topsep}{0ex}       \setlength{\partopsep}{0ex}
      \setlength{\leftmargin}{1em} \setlength{\labelwidth}{1em}
      \setlength{\labelsep}{0.1em} } }
\newcommand{\numsquishend}{\end{list}}

\newcommand{\squishlist}{
   \begin{list}{$\bullet$}
    { \setlength{\itemsep}{-.1ex}      \setlength{\parsep}{0ex}
      \setlength{\topsep}{0ex}       \setlength{\partopsep}{0ex}
      \setlength{\leftmargin}{.8em} \setlength{\labelwidth}{1em}
      \setlength{\labelsep}{0.5em} } }
\newcommand{\squishend}{\end{list}}

\newcommand{\clip}{{\sc Coordination Initiator Inference Problem}\xspace}%

\DeclareMathOperator*{\argmin}{\mathop{\mathrm{argmin}}\limits}

\newcommand{\argmax}{\mathop{\mathrm{argmax}}\limits}

\usepackage{setspace}

\copyrightyear{2020}
\acmYear{2020}
\setcopyright{acmcopyright}
\acmJournal{TKDD}
\acmPrice{15.00}
\acmDOI{10.1145/1122445.1122456}
\acmISBN{978-1-4503-9999-9/18/06}

\begin{document}
\title{Variable-lag Granger Causality and Transfer Entropy for Time Series Analysis}

  
\author{Chainarong Amornbunchornvej}
\orcid{0000-0003-3131-0370}
\affiliation{\institution{National Electronics and Computer Technology Center}
\city{Pathum Thani}
  \country{Thailand}}
\email{chainarong.amo@nectec.or.th}

\author{Elena Zheleva}
\affiliation{\institution{University of Illinois at Chicago}
\city{Chicago}\state{IL}
\country{USA}}
\email{ezheleva@uic.edu}

\author{Tanya Berger-Wolf}
\affiliation{\institution{University of Illinois at Chicago}
\city{Chicago}\state{IL}
\country{USA}}
\affiliation{\institution{The Ohio State University}
\city{Columbus}\state{OH}
\country{USA}}
\email{berger-wolf.1@osu.edu}

\renewcommand{\shortauthors}{C. Amornbunchornvej et al.}

\begin{abstract}
Granger causality is a fundamental technique for causal inference in time series data, commonly used in the social and biological sciences. Typical operationalizations of Granger causality make a strong assumption that every time point of the effect time series is influenced by a combination of other time series with a fixed time delay. The assumption of fixed time delay also exists in Transfer Entropy, which is considered to be a non-linear version of Granger causality.  However, the assumption of the fixed time delay does not hold in many applications, such as collective behavior, financial markets, and many natural phenomena. To address this issue, we develop Variable-lag Granger causality and Variable-lag Transfer Entropy, generalizations of both Granger causality and Transfer Entropy that relax the assumption of the fixed time delay and allow causes to influence effects with arbitrary time delays. In addition, we propose methods for inferring both variable-lag Granger causality and Transfer Entropy relations. In our approaches, we utilize an optimal warping path of Dynamic Time Warping (DTW) to infer variable-lag causal relations. We demonstrate our approaches on an application for studying coordinated collective behavior and other real-world casual-inference datasets and show that our proposed approaches  perform better than several existing methods in both simulated and real-world datasets. Our approaches can be applied in any domain of time series analysis. The software of this work is available in the R-CRAN package: VLTimeCausality.

\end{abstract}

%
%



\maketitle

\section{Introduction}

Inferring causal relationships from data is a fundamental problem in statistics, economics, and science in general. The gold standard for assessing causal effects is running randomized controlled trials which randomly assign a treatment (e.g., a drug or a specific user interface) to a subset of a population of interest, and randomly select another subset as a control group which is not given the treatment, thus attributing the outcome difference between the two groups to the treatment. However, in many cases, running such trials may be unethical, expensive, or simply impossible~\cite{Varian7310}. 
To address this issue, several methods have been developed to estimate causal effects from observational data~\cite{pearl2000causality,Spirtes1993}. 

In the context of time series data, a well-known method that defines a causal relation in terms of \emph{predictability} is Granger causality~\cite{granger1969investigating}. $X$ Granger-causes $Y$ if past information on $X$ predicts the behavior of $Y$ better than $Y$'s past information alone~\cite{Arnold:2007:TCM:1281192.1281203}. In this work, when we refer to causality, we mean specifically the predictive causality defined by Granger causality. The key assumptions of Granger causality are that 1) the process of effect generation can be explained by a set of structural equations, and 2) the current realization  of the effect at any time point is influenced by a set of causes in the past. Similar to other causal inference methods, Granger causality assumes unconfoundedness and that all relevant variables are included in the analysis~\cite{granger1969investigating,peters2017elements}. 

There are several studies that have been developed based on Granger causality ~\cite{liu2012sparse,atukeren2010relationship,peters2013causal}. 

Granger causality is typically studied in the context of linear structural equations. \textit{Transfer Entropy} has been developed as a non-linear extension of Granger causality~\cite{schreiber-prl00,lee2012transfer,PhysRevLett.103.238701}.

The typical operational definitions~\cite{atukeren2010relationship} and inference methods for inferring Granger causality, including the common software implementation packages~\cite{MLSourcecode,RSourcecode}, assume that the effect is influenced by the cause with a fixed and constant time delay. 

 However, the assumption of an effect is fixed-lag influenced by the cause still exists in both Granger causality and transfer entropy. 

This assumption of a fixed and constant time delay between the cause and effect is, in fact, too strong for many applications of understanding natural world and social phenomena. In such domains, data is often in the form of a set of time series and a common question of interest is which time series are the (causal) initiators of patterns of behaviors captured by another set of time series. For example, who are the individuals who influence a group's direction in collective movement? What are the sectors that influence the stock market dynamics right now? Which part of the brain is critical in activating a response to a given action? In all of these cases, effects follow the causal time series with delays that can vary over time~\cite{FLICAtkdd}. The fact that one time series can be caused by multiple initiators and  these initiators can be inferred from time series data~\cite{FLICAtkdd,Arnold:2007:TCM:1281192.1281203}.

To address the remaining gap, we introduce the concepts \emph{Variable-lag Granger causality} and \emph{Variable-lag Transfer Entropy} and methods to infer them in time series data. We prove that our definitions and the proposed inference methods can address the arbitrary-time-lag influence between cause and effect, while the traditional operationalizations of Granger causality, transfer entropy, and their corresponding inference methods cannot. We show that the traditional definitions are indeed special cases of the new relations we define. We demonstrate the applicability of the newly defined causal inference frameworks by inferring initiators of collective coordinated movement, a problem proposed in~\cite{FLICAtkdd}, as well as inferring casual relations in other real-world datasets. 

We use Dynamic Time Warping (DTW)~\cite{sakoe1978dynamic} to align the cause $X$ to the effect time series $Y$ while leveraging the power of Granger causality and transfer entropy. In the literature, there are many clustering-based Granger causality methods that use DTW to cluster time series and perform Granger causality only for time series within the same clusters~\cite{yuan2016deep,Peng:2007:SSI:1288552.1288557}. Previous work on inferring causal relations using both Granger causality and DTW has the assumption that the smaller warping distance between two time series, the stronger the causal relation is~\cite{sliva2015tools}. If the minimum distance of elements within the DTW optimal warping path is below a given distance threshold, then the method considers that there is a causal relation between the two time series. However, their work assumes that Granger causality and DTW run independently. 
In contrast, our method formalizes the integration of Granger causality and DTW by generalizing the definition of Granger causality itself and using DTW as an instantiation of the optimal alignment requirement of the time series. 

In addition to the standard uses of Granger causality and transfer entropy, our methods are capable of:
\squishlist
\item {\bf Inferring arbitrary-lag causal relations:} our methods can infer a causal relation of Granger or transfer entropy where a cause influences an effect with arbitrary delays that can change dynamically; 
\item {\bf Quantifying variable-lag emulation:} our methods can report the similarity of time series patterns between the cause and the delayed effect, for arbitrary delays; 
\squishend

We also prove that when multiple time series cause the behavioral convergence of a set of time series then we can treat the set of these initiating causes in the aggregate and there is a causal relation between this aggregate cause (of the set of initiating time series) and the aggregate of the rest of the time series. We provide  many experiments and examples using both simulated and real-world datasets to measure the performance of our approach in various causality settings and discuss the resulting domain insights. Our framework is highly general and can be used to analyze time series from any domain.

\section{Related work}
Granger causality has inspired a lot of research since its introduction in 1969~\cite{granger1969investigating}. Recent works on Granger causality has focused on various generalizations for it, including ones based on information theory, 
such as transfer entropy~\cite{schreiber-prl00,shibuya-kdd09} and directed information graphs~\cite{7273888}. Recent inference methods are able to deal with missing data~\cite{iseki-aaai19} and enable feature selection~\cite{Sun2015}. Granger causality has even been explored as a method to offer explainability of machine learning models~\cite{schwab-aaai19}. 
However, none of them study tests for Variable-lag Granger causality, as we formalize and propose in this work. 

Many causal inference methods assume that the data is {\em i.i.d.} and rely on knowing a mechanism that generates that data, e.g., expressed through causal graphs or structural equations~\cite{pearl2000causality}. 
In time series data, there are two ways in which time series can be {\em i.i.d.}: 1) the points of one time series are independent of other points in the same time series, 2)  one time series is independent of another time series. Obviously, in most time series, the values of the consecutive time steps violate the {\em i.i.d.} assumption (the first way). In causal inference, the field focuses on the independent between two time series in the second way.

Another set of causal inference methods relax this strong {\em i.i.d} assumption, and instead assume independence between the cause 
and the mechanism 
generating the effect 
~\cite{janzing2010causal,scholkopf2012causal,shajarisales2015telling}. Specifically, knowing a distribution of random variable of cause $X$ never reveals information about the structural function $f(X)$ and vice versa. This idea has been used in the context of times series data~\cite{shajarisales2015telling} by relying on the concept of Spectral Independence Criterion (SIC). If a cause $X$ is a stationary process that generates the effect $Y$ via linear time invariance filter $h$ (mechanism), then $X$ and $h$ should not contain any information about each other but dependency between them and $Y$ exists in spectral sense.

There is a framework of causal inference in~\cite{malinsky2018causal} based on conditional independence tests on time series generated from some discrete-time stochastic processes that allows unknown latent variables. However, the approach in~\cite{malinsky2018causal} still assumes that data points at any time step have been generated from some structural vector autoregression (SVAR). The recent work in~\cite{griveau2019efficient} models causal relation between time series as a form of polynomial function and uses a stochastic block model to find a causal graph. Both works, however, still have the assumption of fixed-lag influence from causes to effects.

Besides, no method studies a causal structure that is unstable\footnote{Unstable causal structures means a relation between effect and causes can be changed overtime. In other words, given time series $X$ causes $Y$, $Y(t)=f(X_1,\dots,X_{t-1})$ and $Y(t')=f'(X_1,\dots,X_{t-1})$ where $t\neq t'$,  $f$ and $f'$ might not be the same. } overtime~\cite{doi:10.1098/rsta.2011.0613}.

Moreover, Transfer Entropy, which is considered to be a non-linear extension of Granger causality~\cite{schreiber-prl00,lee2012transfer,PhysRevLett.103.238701}, still has the fixed-lag assumption.

In our work, we also relax the stationary assumption of time series.


\section{Extension from previous work}
This paper is an extension of our conference proceeding~\cite{VLGranger}. In our previous work~\cite{VLGranger}, we formalized VL-Granger causality and proposed a framework to infer a causal relation using BIC and F-test as main criteria to infer whether $X$ causes $Y$. 

In this work, we formalize \textit{Variable-Lag Transfer Entropy}, which is a non-linear extension of Granger-causality. We investigate the challenge of generalizing Transfer Entropy by relaxing its fixed-lag assumption. Then, we propose a framework to infer VL-Transfer Entropy causal relations.

Moreover, we extend our work on VL-Granger Causality and propose to use a \textit{Bayesian Information Criterion difference ratio} or BIC difference ratio, which is a normalized BIC, as a main criterion. There is evidence that BIC performs better than other model-selection criteria in general~\cite{raffalovich2008model,granger2004forecasting,atukeren2010relationship}.
We also add two new real-world datasets and additional experiments in this current work.

\section{Granger causality and fixed lag limitation}
\begin{figure}[t!]
\centering
\includegraphics[width=.7\columnwidth]{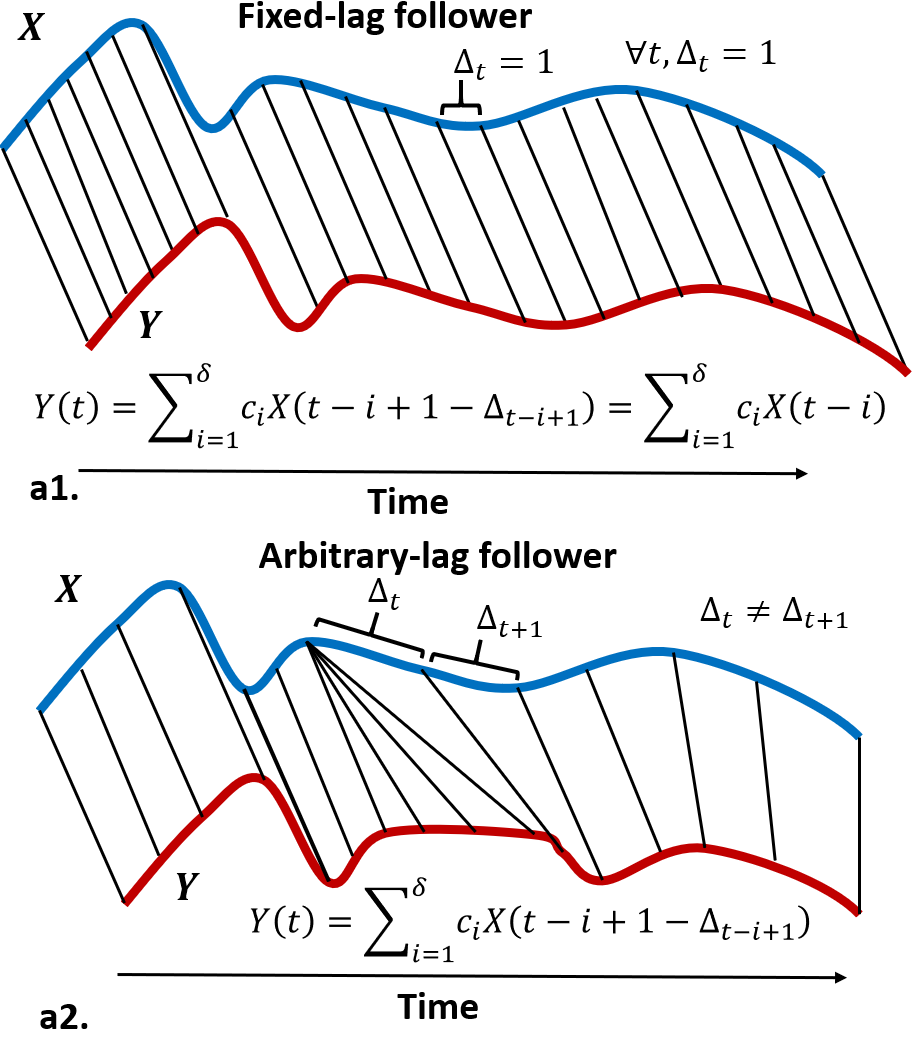}
\caption{(a1-2.) A leader (blue) influences a follower (red) at a specific time point via black lines. (a1.) The follower is a distorted version of a leader with a fixed lag.  (a2.) The follower  is a distorted version of a leader with non-fixed lags in that violates an assumption of Granger causality. Granger causality can handle only the former case and typically fails to handle later case.  We propose the generalization of Granger causality to handle variable-lag situation (equation in a2.).  }
\label{fig:LinearVSArbLag}
\end{figure}

Let $X=(X(1),\dots,X(t),\dots)$ be a time series. We will use $X(t)\in \mathbb{R}$ to denote an element of $X$ at time $t$. Given two time series $X$ and $Y$, it is said that  $X$ Granger-causes~\cite{granger1969investigating}  $Y$ if the information of $X$ in the past helps  improve the prediction of the behavior of $Y$, over $Y$'s past information alone~\cite{Arnold:2007:TCM:1281192.1281203}. The typical way to operationalize this general definition of  Granger causality ~\cite{atukeren2010relationship} is to define it as follows:
\begin{definition}[Granger causal relation]
\label{def:GC}
Let $X$ and $Y$ be time series, and $\delta_{max}\in\mathbb{N}$ be a maximum time lag. We define two residuals of regressions of $X$ and $Y$,  $r_{Y}, r_{YX}$, below: 
\begin{equation}\label{eq1}
	r_{Y}(t) = Y(t) - \sum_{i=1}^{\delta_{max}}a_i Y(t-i),
\end{equation}
\begin{equation}\label{eq2}
	r_{YX}(t) = Y(t)- \sum_{i=1}^{\delta_{max}} (a_i Y(t-i) +  b_i X(t-i)),
\end{equation}
where $a_i$ and $b_i$ are constants that optimally minimize the residual from the regression. Then $X$ Granger-causes $Y$ if the variance of $r_{YX}$ is less than  the variance of $r_Y$.
\end{definition}

This definition assumes that, for all $t>0$, $Y(t)$ can be predicted by the fixed linear combination of $a_1Y(1),\dots,a_{t-\Delta}Y(t-\Delta)$ and $b_1X(1),\dots,b_{t-\Delta}X(t-\Delta)$ with some fixed $\Delta>0$ and every $a_i,b_i$ is a fixed constant over time~\cite{atukeren2010relationship,Arnold:2007:TCM:1281192.1281203}. However, in reality, two time series might influence each other with a sequence of arbitrary, non-fixed time lags. For example, Fig.~\ref{fig:LinearVSArbLag}(a2.) has $X$ as a cause time series and $Y$ as the effect time series that imitates the values of $X$ with arbitrary lags.  Because $Y$ is not affected by $X$ with a fixed lags and the linear combination above can  change over time,  the standard Granger causality tests cannot appropriately infer Granger-causal relation between $X$ and $Y$ even if $Y$ is just a slightly distorted version of $X$ with some lags. For a concrete example, consider a movement context where time series represent  trajectories. Two people follow each other if they move in the same trajectory. Assuming the followers follow leaders with a fixed lag means the followers walk lockstep with the leader, which is not the natural way we walk.  Imagine two people embarking on a walk. The first starts walking, the second catches up a little later. They may walk together for a bit, then the second stops to tie the shoe and catches up again. The delay between the first and the second person keeps changing, yet there is no question the first sets the course and is the cause of the second's choices where to go. Fig.~\ref{fig:LinearVSArbLag} illustrates this example. 


\section{Variable-lag Granger Causality}

Here, we propose the concept of variable-lag Granger causality,  \emph{VL-Granger causality} for short, which generalizes the Granger causal relation of Definition~\ref{def:GC}  in a way that addresses the fixed-lag limitation. We demonstrate the application of the new causality relation for a specific application of inferring initiators and followers of collective behavior.

 \begin{definition}[Alignment of time series]
 \label{def:AlignSeq}
 An alignment between two time series $X$ and $Y$ is a sequence of pairs of indices $(t_i,t_j)$, aligning $X(t_i)$ to $Y(t_j)$, such that for any two pairs in the alignment $(t_{i},t_{j})$ and $(t'_{i},t'_{j})$, if $t_{i}< t'_{i}$ then $t_{j} < t'_{j}$ (non-crossing condition). The alignment defines a sequence of delays $P=(\Delta_1,\dots,\Delta_t,\dots)$, where $\Delta_t\in\mathbb{Z}$ and $X(t-\Delta_t)$ aligns to $Y(t )$.
 \end{definition}

\begin{definition}[VL-Granger causal relation]
\label{def:ArbCausalR}
Let $X$ and $Y$ be time series, and $\delta_{max}\in\mathbb{N}$ be a maximum time lag (this is an upper bound on the time lag between any two pairs of time series values to be considered as causal). We define residual $r^*_{YX}$ of the regression:
\begin{equation}\label{eq3}
	r^*_{YX}(t) = Y(t)- \sum_{i=1}^{\delta_{max}} (a_i Y(t-i) +  b_i X(t-i) + c_i X^*(t-i)).
\end{equation}
Here $X^*(t-i)=X(t-i+1-\Delta_{t-i+1} )$, where $\Delta_{t} > 0$ is a time delay constant in  the optimal alignment sequence $P^*$ of $X$ and $Y$ that minimizes the residual of the regression.  The constants $a_i,b_i$, and $c_i$ optimally minimize the residuals $r_Y$, $r_{YX}$, and $r^*_{YX}$, respectively. The terms $b_i$ and $c_i$ can be combined but we keep them separate to clearly denote the difference between the original and proposed VL-Granger causality.
We say that $X$ VL-Granger-causes $Y$ if the variance of $r^*_{YX}$ is less than  the variances of both $r_Y$ and $r_{YX}$.
\end{definition} 

In order to make Definition~\ref{def:ArbCausalR} fully operational for this more general case (and to find the optimal constants values), we need a similarity function between two sequences which will define the optimal alignment. We propose such a similarity-based approach in Definition~\ref{def:VLemulation}.  Before defining this approach, we show that VL-Granger causality is the proper generalization of the traditional operational definition of Granger causality stated in Definition~\ref{def:GC}.
Clearly, assuming that all delays are less than $\delta_{max}$, if all the delays are constant, then $r^*_{YX}(t)=r_{YX}(t)$.  
\begin{proposition}
\label{prop:VLGrC1}
Let $X$ and $Y$ be time series and $P$ be their alignment sequence. If   $\forall t,\Delta_t=\Delta$, then $r^*_{YX}(t)=r_{YX}(t)$.
\end{proposition}


We must also show that the variance of $r^*_{YX}(t)$ is no greater than the variance of $r_{YX}(t)$.
\begin{proposition}
\label{prop:VLGrC2}
Let $X$ and $Y$ be time series, $P=(\Delta_1, \dots,\Delta_t, \dots)$ be their alignment sequence such that $Y(t)=X(t-\Delta_t)$. If  $\exists \Delta_t,\Delta_{t'}\in P$,  such that $\Delta_t \neq \Delta_{t'}$ and  $\forall t, X(t)\neq X(t-1)$, then   $VAR(r^*_{YX}) < VAR(r_{YX} )$.
 \end{proposition}
\begin{proof}
Because $Y(t)=X(t-\Delta_t)$, by setting $a_i=0,b_i=0,c_i=1$ for all $i$, we have  $r^*_{YX}=0$. In contrast, suppose $\Delta_{t+1}=\Delta_t+1$ and $X(t-\Delta_t-1)\neq X(t-\Delta_t) \neq X(t-\Delta_t+1)$, so $Y(t)=Y(t+1)=X(t-\Delta_t)$. Because  $a_i,b_i$ must be constant for all time step $t$ to compute $r_{YX}(t)$, at time $t$, the regression must choose to match either 1) $Y(t) - X(t-\Delta_t)=0$ and $Y(t+1) - X(t+1-\Delta_t)\neq 0$ or 2) $Y(t) - X(t-\Delta_{t+1})\neq 0$ and $Y(t+1) - X(t+1-\Delta_{t+1})=0$. Both 1) and 2) options make  $r_{YX}(t) +  r_{YX}(t+1) >0$. Hence, $VAR(r^*_{YX}) < VAR(r_{YX} )$.
\end{proof}
According to Propositions \ref{prop:VLGrC1} and \ref{prop:VLGrC2}, VL-Granger causality is the generalization of the Def.~\ref{def:GC} and always has lower or equal variance.


Of a particular interest is the case when there is an explicit similarity relation defined over the domain of the input time series. The underlying alignment of  VL-Granger causality then should incorporate that similarity measure and the methods for inferring the optimal alignment for the given similarity measure.

\begin{definition}[Variable-lag emulation]
\label{def:VLemulation}
Let $\mathcal{U}$ be a set of time series, $X,Y\in \mathcal{U}$, and $\mathrm{sim}: \mathcal{U}\times \mathcal{U} \to [0,1]$ be a similarity measure between two time series.

For a threshold $\sigma \in (0,1]$, if there exists a sequence of numbers $P=(\Delta_1,\dots,\Delta_t,\dots)$ s.t. $\mathrm{sim}(\tilde{X},Y) \geq \sigma$ when $\tilde{X}(t)=X(t-\Delta_t)$, then we use the following notation: 
\squishlist 
  \item if $\forall \Delta_t\in P, \:\Delta_t \geq 0$ , then $Y$ emulates $X$, denoted by  ${X \preceq Y}$,
  \item if $\forall \Delta_t\in P,  \:\Delta_t \leq  0$ , then $X$ emulates $Y$, denoted by ${Y \preceq X}$,
  \item if  ${X \preceq Y}$ and ${Y \preceq X}$, then ${Y \equiv X}$.
\squishend
We denote ${X \prec Y}$ if ${X \preceq Y}$ and $\exists \Delta_t\in P, \Delta_t > 0$.
\label{def:follr}
\end{definition}

Note, here the $\mbox{sim}$ similarity function does not have to be a distance function that obeys, among others, a triangle inequality. It can be any function that quantitatively compares the two time series. For example, it may be that when one time series increases the other decreases. We provide a more concrete and realistic example in the application setting below.

Adding this similarity measure to Definition~\ref{def:ArbCausalR} allows us to instantiate the notion of the optimal alignment $P^*$ as the one that maximizes the similarity between $X$ and $Y$: 
\begin{equation}
\label{eq:Plag}
    P^*=\argmax_P \mathrm{sim}(\tilde{X},Y),
\end{equation}
where  $\tilde{X}(t)=X(t-\Delta_t)$ for any given $P$ and  $\Delta_t \in P$.  With that addition, if $X\prec Y$, then $X$ VL-Granger-causes $Y$. This allows us to operationalize VL-Granger causality by checking for variable-lag emulation, as we describe in the next section.

\subsection{Example application: Initiators and followers}

In this section, we demonstrate an application of the VL-Granger causal relation to finding initiators of collective behavior. The Variable-lag emulation concept corresponds to a relation of following  in the leadership literature~\cite{FLICAtkdd}. That is,  $X\prec Y$ if $Y$ is a {\em follower} of $X$.  We are interested in the phenomenon of group convergence to a  consensus behavior and answering the question of which subset of individuals, if any, initiated that collective consensus behavior.
With that in mind,  we now define the concept of an initiator and provide a set of subsidiary definitions that allow us to formally show  (in Proposition~\ref{prop:LFprop}) that initiators of  collective behavior are indeed the time series that VL-Granger-cause the collective pattern in the set of the time series. In order to do this, we generalize our two-time series definitions to the case of multiple time series by defining the notion of an aggregate time series, which is consistent with previous Granger causality generalizations to multiple time series~\cite{siggiridou2016granger,doi:10.1098/rsta.2011.0613,chen2004analyzing}.

\begin{definition} [Initiators] Let $\mathcal{U}=\{U_{1},\dots,U_{n}\}$ be a set of time series. We say that $\mathcal{X}\subseteq\mathcal{ U}$ is a set of initiators if $\forall U\in\mathcal{U\setminus\mathcal{X}}$,  $\: \exists X \in \mathcal{X}$,  $\: s.t. X \prec U$, and, conversely,  $\forall X\in  \mathcal{X} \: \exists U\in \mathcal{U\setminus\mathcal{X}},  \:  s.t.  X \prec U$. That is, every time series follows some initiator and every initiator has at least one follower.
\end{definition}

Given a set of time series $\mathcal{U}=\{U_{1},\dots,U_{n}\}$ , and a set of time series $\mathcal{X}\subseteq\mathcal{U}$, we can define an aggregate time series as a time series of means at each step:

\begin{equation}
agg(\mathcal{X})=\left({\frac{1}{|\mathcal{X}|}\sum_{U\in\mathcal{X}}U(0),\dots,\frac{1}{|\mathcal{X}|}\sum_{U\in\mathcal{X}}U(t),\dots}\right)
\end{equation}

In order to identify the state of reaching a collective consensus of a time series, while allowing for some noise, we adopt the concept of  $\epsilon$-convergence  from~\cite{doi:10.1137/100791671}.

\begin{definition} [$\epsilon$-convergence] Let $Q$ and $U$ be time series, $dist:\mathbb{R}^2\times [0,1]$ be a distance function,  and $0<\epsilon\leq1/2$.  If for all time $t\in [t_{0},t_{1}], \: dist(Q(t),U(t))\leq\epsilon$, then $Q$ and $U$ $\epsilon$-converge toward each other in the interval $[t_0,t_1]$. If $t_1=\infty$ then we  say that $Q$ and $U$ $\epsilon$-converge at time $t_0$. 
\end{definition}


\begin{definition} [$\epsilon$-convergence coordination set] Given a set of time series $\mathcal{U}=\{U_{1},\dots,U_{n}\}$, if all time series in $\mathcal{U}$ $\epsilon$-converge toward $agg(\mathcal{U})$, then we say that the set $\mathcal{U}$ is an $\epsilon$-convergence coordination set.
\end{definition}

We are finally ready to state the main connection between initiation of collective behavior and VL-Granger causality.

\begin{proposition}
\label{prop:LFprop}
Let  $dist:\mathbb{R}^2\times [0,1]$ be a distance function, $\mathcal{U}$ be a set of time series, and  $\mathcal{X}\subseteq  \mathcal{U}$ be a set of initiators, which is an $\epsilon$-convergence coordination set converging towards $agg(\mathcal{X})$ in the interval $[t_0,t_1]$. For any $U, U'\in \mathcal{U}$ of length $T$, let  $$\mathrm{sim}(U,U')=\frac{\sum_t 1- dist(U(t),U'(t))}{T}. $$ 
If  for any  $U, U' \in \mathcal{U}$ their similarity $\mathrm{sim}(U,U')\geq 1-\epsilon$ in the interval $[t_0,t_1]$, then $agg(\mathcal{X})$  VL-Granger-causes $agg(\mathcal{U}\setminus\mathcal{X})$ in that interval.
\end{proposition}
 \begin{proof}
  
  Suppose  $\forall X\in \mathcal{X}$, $X$ and $agg(\mathcal{X})$ $\epsilon$-converge toward each other in the interval $[t_0,t_1]$, then, by definition, for all the times $t \in [t_{0},t_1], \: dist(agg(\mathcal{X})(t),X(t))\leq\epsilon$. By the definition of initiators,  $\forall U\in\mathcal{U}\setminus\mathcal{X}, \: \exists X \in \mathcal{X}$, such that $X\prec U$, from some time $t_2 > t_0$. Thus, we have $\forall t$, s. t.  $t_2 \leq t\leq t_1, \: \: dist(X(t),U(t))\leq\epsilon$, which means $dist(agg(\mathcal{X}),U(t))\leq 2\epsilon$. Hence, we have    $\forall t, t_2\leq t \leq t_1, \:\: dist(agg(\mathcal{X})(t),agg(\mathcal{U}\setminus\mathcal{X})(t) )\leq 2\epsilon$. Since $agg(\mathcal{X})$ $2\epsilon$-converges towards some constant line $v$ in the interval $[t_0,t_1]$ and $agg(\mathcal{U}\setminus\mathcal{X})(t) )$  $2\epsilon$-converges towards the same line $v$ in the interval $[t_2, t_1]$, hence $agg(\mathcal{X}) \prec agg(\mathcal{U}\setminus\mathcal{X})$, which means, by definition, that  $agg(\mathcal{X})$  VL-Granger-causes $agg(\mathcal{U}\setminus\mathcal{X})$.
 \end{proof}
 
We have now shown that a subset of time series are initiators of a pattern of collective behavior  of an entire set if that subset VL-Granger-causes the behavior of the set. Thus, VL-Granger causality can solve the  \clip~\cite{FLICAtkdd}, which is a problem of determining  whether a pattern of collective behavior was spurious or instigated by some subset of initiators and, if so, finding those initiators who initiate collective patterns that everyone follows.

\section{Variable-lag Transfer Entropy Causality}
In this section, we generalize our concept of VL-Granger causality to the  non-linear extension of Granger causality, \textit{Transfer Entropy}~\cite{lee2012transfer,PhysRevLett.103.238701}. Given two time series $X$ and $Y$, and a probability function $p(\cdot)$, the \textit{Transfer Entropy} from $X$ to $Y$ is defined as follows:

\begin{equation}
\label{eq:trnsEnp}
    \mathcal{T}_{X\xrightarrow{} Y} = H(Y(t)\mid Y_{t-1}^{(k)} ) - H(Y(t)\mid Y_{t-1}^{(k)},X_{t-1}^{(l)} ).
\end{equation}

Where $H(\cdot\mid\cdot)$ is a conditional entropy, $k,l$ are lag constants,  $Y_{t-1}^{(k)}=Y(t-1),\dots,Y(t-k)$, and $X_{t-1}^{(l)}=X(t-1),\dots,X(t-l)$. 

One of the most common types of entropy is Shannon entropy~\cite{6773024}, based on which the function $H(\cdot)$ is defined as

\begin{equation}
    H(X) = - \sum_{t} p(X(t))\text{log}_2 \left( p(X(t))\right).
\end{equation}
Based on this function, the Shannon transfer entropy~\cite{BEHRENDT2019100265,lee2012transfer} is:

\begin{equation}
\label{eq:shtrnsEnp}
    \mathcal{T}_{X\xrightarrow{} Y} =\sum  p(Y_{t}^{(k)},X_{t-1}^{(l)})\text{log}_2 \frac{p( Y(t)\mid Y_{t-1}^{(k)},X_{t-1}^{(l)}) }{p( Y(t)\mid Y_{t-1}^{(k)}) } .
\end{equation}

Typically, we infer whether $X$ causes $Y$ by comparing $\mathcal{T}_{X\xrightarrow{} Y}$ and $\mathcal{T}_{Y\xrightarrow{} X}$. If  $\mathcal{T}_{X\xrightarrow{} Y}>\mathcal{T}_{Y\xrightarrow{} X}$, then we state that $X$ causes $Y$. However, transfer entropy is also limited by the fixed-lag assumption. 
Equation~\ref{eq:trnsEnp} shows a comparison between $Y(t)$ and $Y_{t-1}^{(k)}$ and $X_{t-1}^{(l)}$ and no variable lags are allowed. Therefore, we formalize the \textit{Variable-lag Transfer Entropy} or VL-Transfer entropy function as below:
\begin{equation}
\label{eq:VLtrnsEnp1}
     \mathcal{T}^\text{VL}_{X\xrightarrow{} Y}(P)= H(Y(t)\mid Y_{t-1}^{(k)} ) - H(Y(t)\mid Y_{t-1}^{(k)},\tilde{X}_{t-1}^{(l)} )
\end{equation}
Where  $\tilde{X}_{t-1}^{(l)}=X(t-1-\Delta_{t-1}),\dots,X(t-l-\Delta_{t-l})$ for a given $P$,  $\Delta_t \in P$, and , $\Delta_t > 0$.

\begin{proposition}
\label{prop:VLTE}
Let $X$ and $Y$ be time series and $P$ be their alignment sequence. If   $\forall \Delta_t \in P,\Delta_t=0$, then $\mathcal{T}^\text{VL}_{X\xrightarrow{} Y}(P) = \mathcal{T}_{X\xrightarrow{} Y}$.
\end{proposition}
\begin{proof}
By setting $\Delta_t =0$ for all $t$, the function $\mathcal{T}^\text{VL}_{X\xrightarrow{} Y}(P)$ in Eq.~\ref{eq:VLtrnsEnp1} is equal to $\mathcal{T}_{X\xrightarrow{} Y}$ in Eq.~\ref{eq:trnsEnp}.
\end{proof}

 Hence,  \textit{Variable-lag Transfer Entropy} function generalizes the transfer entropy function.  To find an appropriate $P$, we can use $P^*$ in Eq.~\ref{eq:Plag} that is a result of alignment of time series $X$ along with $Y$. The  $P^*$ in Eq.~\ref{eq:Plag} represents a sequence of time delay that matches the most similar pattern of time series $X$ with the pattern in time series $Y$ where the pattern of $X$ comes before the pattern of $Y$.






\section{VL-Granger and VL-Transfer Entropy Causality Inference}

\subsection{Variable-lag Causality Inference}
Given a target time series $Y$, a candidate causing time series $X$, a threshold $\sigma$, a significance threshold $\alpha$ (or other threshold if we do not use statistical testing), the max lag $\delta_{max}$, and the linear flag $linearFLAG$, our framework evaluates whether $X$ variable-lag causes $Y$, $X$ fixed-lag causes $Y$ or no conclusion of causation between $X$ and $Y$ using either Granger causality or Transfer Entropy, which is a non-linear extension of Granger causality. In Algorithm~\ref{algo:MainFunc}, users can set either $linearFLAG=true$ to run Granger causality or $linearFLAG=false$ for Transfer Entropy. \\

 For  $linearFLAG=true$, in Algorithm~\ref{algo:MainFunc} line 2-3,  we have a fix-lag parameter $FixLag$ that controls whether we choose to compute the normal Granger causality ($FixLag=true$) or VL-Granger causality ($FixLag=false$). For  $linearFLAG=false$, in the line 5-6, we compute Transfer Entropy if $FixLag=true$. Otherwise, we compute whether $X$ causes $Y$ w.r.t. VL-Transfer Entropy.

  We present the high level logic of the algorithm. However, the actual implementation is more efficient by removing the redundancies of the presented logic.\\

For  $linearFLAG=true$, first, we compute Granger causality (line 2 in  Algorithm~\ref{algo:MainFunc}) using a function in Section~\ref{sec:opVLGranger}. The flag $fixLagResult=true$ if $X$ Granger-causes $Y$, otherwise $fixLagResult=false$. Second, we  compute VL-Granger causality (line 3 in  Algorithm~\ref{algo:MainFunc}). The flag $VLResult=true$ if $X$ VL-Granger-causes $Y$, otherwise, $VLResult=false$. Third, in line 4 in  Algorithm~\ref{algo:MainFunc}, based on the work in~\cite{atukeren2010relationship}, we use the Bayesian Information Criteria  (BIC) to compare the residual of regressing $Y$ on $Y$ past information, $r_Y$, with the residual of regressing $Y$ on $Y$ and $X$ past information $r_{YX}$. We use $v_1 \ll v_2$ to represent that $v_1$ is less than $v_2$ with statistical significance by using some statistical test(s) or criteria. If $BIC(r_Y)\ll BIC(r_{YX})$, then we conclude that the prediction of  $Y$  using $Y,X$ past information is better than the prediction of  $Y$  using $Y$ past information alone. For this work, to determine $BIC(r_Y)\ll BIC(r_{YX})$, we use \textit{Bayesian Information Criterion difference ratio} (see Section~\ref{sec:BICDiffratio}). If $BIC(r_Y)\ll BIC(r_{YX})$, then $VLflag=true$, otherwise, $VLflag=false$. \\
  
For  $linearFLAG=false$, first, we compute Transfer Entropy causality (line 5 in  Algorithm~\ref{algo:MainFunc}) using a function in Section~\ref{sec:opVLTE}. The flag $fixLagResult=true$ if $X$ causes $Y$ in Transfer Entropy, otherwise, $fixLagResult=false$. Second, we  compute VL-Transfer-Entropy causality (line 6 in  Algorithm~\ref{algo:MainFunc}). The flag $VLResult=true$ if $X$ causes $Y$ in VL-Transfer Entropy, otherwise, $VLResult=false$. To determine whether $X$ causes $Y$ in Transfer Entropy, we use the \textit{Transfer Entropy Ratio} (see Section~\ref{sec:VLTEmethod}).

In line 7, if the normal Transfer Entropy ratio is less than the VL-Transfer Entropy ratio, then $VLflag=true$, otherwise, $VLflag=false$. \\

Note that $VLflag=true$ when the result of variable-lag version is better than the fixed-lag version in both Granger causality and Transfer Entropy.

Using the results of  $fixLagResult$, $VLResult$, and $VLflag$, we proceed to report the conclusion of causal relation between $X$ and $Y$ w.r.t.  the following four conditions.
 
 \squishlist
\item {\bf If both $fixLagResult$ and $VLResult$ are true}, then we determine $VLflag$. If $VLflag=true$, then we conclude that $X$ causes $Y$ with variable lags, otherwise, $X$ causes $Y$ with a fix lag (line 9 in  Algorithm~\ref{algo:MainFunc}).
\item {\bf If $fixLagResult$ is true but $VLResult$ is false}, then we conclude that $X$ causes $Y$ with a fix lag (line 10 in Algorithm~\ref{algo:MainFunc}).
\item {\bf If $fixLagResult$ is false but $VLResult$ is true}, then we conclude that $X$ causes $Y$ with variable lags (line 11 in Algorithm~\ref{algo:MainFunc}).

\item {\bf If both $fixLagResult$ and $VLResult$ are false}, then we cannot conclude whether $X$ causes $Y$ (line 12 in Algorithm~\ref{algo:MainFunc}).
\squishend

\setlength{\intextsep}{0pt}
\IncMargin{1em}
\begin{algorithm2e}
\caption{Time-lag test function}
\label{algo:MainFunc}
\SetKwInOut{Input}{input}\SetKwInOut{Output}{output}
\Input{ $X,Y$, $\sigma$, $\gamma$  (or $\alpha$), $\delta_{max}$, $linearFLAG$ } 
\Output{$XCausesY$}
\begin{small}
\SetAlgoLined
\nl \uIf{$linearFLAG=true$}
{
    \nl ($fixLagResult$,$r_{Y},r_{YX}$)=VLGrangerFunc($X,Y$, $\sigma$, $\gamma$, $\delta_{max}$, $FixLag = true$)\;
    \nl ($VLResult$,$r_{Y},r_{DTW}$)= VLGrangerFunc($X,Y$, $\sigma$, $\gamma$, $\delta_{max}$, $FixLag = false$)\;
    \nl $VLflag= \big(BIC(r_{DTW}) \ll min(BIC(r_{YX}),BIC(r_{Y}) ) \big)$\;
}  \Else{
\nl ($fixLagResult$,$\mathcal{T}_{X\xrightarrow{} Y},\mathcal{T}_{Y\xrightarrow{} X}$)=VLTransferEFunc($X,Y$, $\delta_{max}$, $FixLag = true$)\;
\nl ($VLResult$,$\mathcal{T}^\text{VL}_{X\xrightarrow{} Y},\mathcal{T}^\text{VL}_{Y\xrightarrow{} X}$)=VLTransferEFunc($Y,X$, $\delta_{max}$, $FixLag = false$)\;

\nl $VLflag=\mathcal{T}(X,Y)_\textrm{ratio} < \mathcal{T}^\text{VL}(X,Y)_\textrm{ratio}$\;

}
\nl \uIf{$fixLagResult=true$ }
{
                 \uIf{$VLResult=true$}
                {
                
               \nl   \uIf{$VLflag = true$}
                {
                $XCausesY$ = TRUE-VARIABLE\; 
                }
                 \Else{
                $XCausesY$ = TRUE-FIXED\; 
                }              
                
                }
                \Else{
                \nl  $XCausesY$ = TRUE-FIXED\; 
                }
} 
 \Else{
  				\uIf{$VLResult=true$}
                {
                 \nl  $XCausesY$ = TRUE-VARIABLE\; 
                }
                 \Else{
                \nl  $XCausesY$ = NONE\; 
                }
}
\nl return $XCausesY$\;
\end{small}
\end{algorithm2e}\DecMargin{1em}

Note that we assume the maximum lag value $\delta_{\max}$ is given as an input, as it is for all definitions of both Granger causality and Transfer Entropy. For practical purposes, a value of a large fraction ({\em e.g., } half) of the length of the time series can be used. However, there is, of course, a computational trade-off between the magnitude of $\delta_{\max}$ and the time it takes to compute both Granger causality and Transfer Entropy.

\subsection{VL-Granger causality operationalization}
\label{sec:opVLGranger}
Next, we describe the details of the VL-Granger function used in Algorithm~\ref{algo:MainFunc}: line 1-2. Given two time series $X$ and $Y$, a threshold $\gamma$ (or a significance level $\alpha$ if we use F-test), the maximum possible lag $\delta_{max}$, and whether we want to check for variable or fixed lag $FixLag$, Algorithm~\ref{algo:VLGrangerCalFunc} reports whether $X$ causes $Y$ by setting $GrangerResult$ to be true or false, and by reporting on two residuals  $r_{Y}$ and $r_{YX}$.

First, we compute the residual $r_Y$ of regressing of $Y$ on $Y$'s information in the past (line 1). Then, we regress $Y(t)$ on $Y$ and $X$ past information to compute the residual $r_{YX}$ (line 2). If $BIC(r_{YX}) \ll BIC(r_{Y})$, then $X$ Granger-causes $Y$ and we set $GrangerResult=true$ (line 7). 
If $FixLag$ is true, then we report the result of typical Granger causality. Otherwise, we consider VL-Granger causality (lines 3-5) by computing the emulation relation between $X^{DTW}$ and $Y$ where $X^{DTW}$ is a version of $X$ that is reconstructed through DTW and is most similar to $Y$, captured by $DTWReconstructionFunction(X,Y)$ which we explain in Section~\ref{sec:DTWRecont}. 

Afterwards, we do the regression of $Y$ on $X^{DTW}$'s past information to compute  residual $r_{DTW}$ (line 4). Finally, we check whether $BIC(r_{DTW}) \ll BIC(r_{Y})$ (line 6-9) (see Section~\ref{sec:BICDiffratio}). If so, $X$ VL-Granger-causes $Y$. Additionally, after running $DTWReconstructionFunction(X,Y)$, we might check the condition $simValue\geq\sigma$ in order to claim that whether $X$ VL-Granger-causes $Y$ and ${X \preceq Y}$. 

In the next section, we describe the details of how to construct $X^{DTW}$ and how to estimate the emulation similarity value $simValue$.

\setlength{\intextsep}{0pt}
\IncMargin{1em}
\begin{algorithm2e}
\caption{VLGrangerFunc}
\label{algo:VLGrangerCalFunc}
\SetKwInOut{Input}{input}\SetKwInOut{Output}{output}
\Input{$X,Y$, $\delta_{max}$,  $\sigma$, $\gamma$ (or $\alpha$), $FixLag$ }
\Output{ $GrangerResult$,$r_{Y},r_{YX}$}
\begin{small}
\SetAlgoLined
 \nl Regress $Y(t)$ on $Y(t-\delta_{max}),\dots, Y(t-1)$, then compute the residual $r_{Y}(t)$\;
 
  \uIf{$FixLag$ is true}{
   \nl Regress $Y(t)$ on $Y(t-\delta_{max}),\dots, Y(t-1)$ and $X(t-\delta_{max}),\dots, X(t-1)$, then compute the residual  $r_{YX}(t)$\;
  }
 \Else{
 \nl $X^{DTW}$,$simValue$ = DTWReconstructionFunction($X,Y$) \;
 \nl  Regress $Y(t)$ on $Y(t-\delta_{max}),\dots, Y(t-1)$ and $X^{DTW}(t-\delta_{max}),\dots, X^{DTW}(t-1)$, then compute the residual  $r_{DTW}$\;
 \nl $r_{YX}=r_{DTW}$\;
 } 
 
 \nl \uIf{$BIC_1(r_{YX}) \ll BIC_0(r_{Y})$  } 
{
\nl $GrangerResult=true$
}
\nl \Else{
\nl $GrangerResult=false$ \;}
\nl return $GrangerResult$,$r_{Y},r_{YX} $\;
\end{small}
\end{algorithm2e}\DecMargin{1em}

\subsection{Dynamic Time Warping for inferring VL-Granger causality.}
\label{sec:DTWRecont}

In this work, we propose to use Dynamic Time Warping (DTW)~\cite{sakoe1978dynamic}, which is a standard distance measure between two time series.  DTW calculates the distance between two time series by aligning sufficiently similar patterns between two time series, while allowing for local stretching (see Figure~\ref{fig:LinearVSArbLag}). Thus, it is particularly well suited for calculating the variable lag alignment. 

Given time series $X$ and $Y$, Algorithm~\ref{algo:DTWReFunc} reports reconstructed time series $X^{DTW}$ based on $X$ that is most similar to $Y$, as well as the emulation similarity $simValue$ between the two series. 
First, we use $DTW(X,Y)$ to find the optimal alignment sequence $\hat{P}=(\Delta_1,\dots,\Delta_t,\dots)$ between $X$ and $Y$, as defined in Definition~\ref{def:AlignSeq}. Efficient algorithms for computing $DTW(X,Y)$ exist and they can incorporate various kernels between points~\cite{Mueen:2016:EOP:2939672.2945383,sakoe1978dynamic}. Then, we use $\hat{P}$ to construct $X^{DTW}$ where $X^{DTW}(t)=X(t-\Delta_t)$. However, we also use cross-correlation to normalize  $\Delta_t$ since DTW is sensitive to a noise of alignment (Algorithm~\ref{algo:DTWReFunc} line 3-5). 

Afterwards, we use $X^{DTW}$ to predict $Y$ instead of using only $X$ information in the past in order to infer a VL-Granger causal relation in Definition~\ref{def:ArbCausalR}. The benefit of using DTW is that it can match time points of $Y$ and $X$ with non-fixed lags (see Figure~\ref{fig:LinearVSArbLag}). Let $\hat{P}=(\Delta_1,\dots,\Delta_t,\dots)$ be the DTW optimal warping path of $X,Y$ such that for any $\Delta_t \in \hat{P}$, $Y(t)$ is most similar to $X(t-\Delta_t)$. 

In addition to finding $X^{DTW}$, $DTWReconstructionFunction$ estimates the emulation similarity $simValue$ between $X,Y$ in line 3. For that, we adopt the measure from~\cite{FLICAtkdd} below: 


\begin{equation}
	\mathrm{s}(\hat{P})=\frac{\sum_{\Delta_t \in \hat{P}}\mathrm{sign}(\Delta_t)}{|\hat{P}|},
	\label{eq:traCorr}
\end{equation}
where $0<\mathrm{s}(\hat{P})\leq 1$ if $X\preceq Y$,  $-1\leq \mathrm{s}(\hat{P})<0$ if $Y\preceq X$, otherwise zero. Since the $\mathrm{sign}(\Delta_t)$ represents whether $Y$ is similar to $X$ in the past ($\mathrm{sign}(\Delta_t) >0$) or  $X$ is similar to $Y$ in the past ($\mathrm{sign}(\Delta_t) <0$), by comparing the sign of $\mathrm{sign}(\Delta_t)$, we can infer whether $Y$ emulates $X$. The function $\mathrm{s}(\hat{P})$ computes the average sign of   $\mathrm{sign}(\Delta_t)$ for the entire time series. If $\mathrm{s}(\hat{P})$ is positive, then, on average, the number of times that $Y$ is similar to $X$ in the past is greater than the number of times that $X$ is similar to some values of $Y$ in the past.  Hence, $\mathrm{s}(\hat{P})$ can be used as a proxy to determine whether $Y$ emulates $X$ or vice versa.  We use \textit{dtw} R package~\cite{giorgino2009computing} for our DTW function. For more details regarding DTW, please see Appendix~\ref{apdx:DTW}.

\setlength{\intextsep}{0pt}
\IncMargin{1em}
\begin{algorithm2e}
\caption{ DTWReconstructionFunction}
\label{algo:DTWReFunc}
\SetKwInOut{Input}{input}\SetKwInOut{Output}{output}
\Input{ $X,Y$}
\Output{ $X^{DTW}$, $simValue$}
\begin{small}
\SetAlgoLined
\nl $\hat{P}=(\Delta_1,\dots,\Delta_t,\dots)$ = DTWFunction( $X,Y$)  \tcp{Getting the warping path from Algorithm~\ref{algo:DTWDistFunc}}
\nl $\hat{P}_0=(\Delta_0,\dots,\Delta_0,\dots)$=CrossCorrelation($X,Y$)\;
\nl \For{ all $t$}{
 \uIf{$DIST(X(t-\Delta_t),Y(t))<DIST(X(t-\Delta_0),Y(t))$}
 {
 \nl set  $X^{DTW}(t-1)=X(t-\Delta_t)$ and $\hat{P}^*(t)=\Delta_t$\;
 }\Else{
 \nl set  $X^{DTW}(t-1)=X(t-\Delta_0)$ and $\hat{P}^*(t)=\Delta_0$ \;
 }
}
\nl        $simValue =\mathrm{s}(\hat{P}^*)$ \;
Return $X^{DTW}$, $simValue$\;
\end{small}
\end{algorithm2e}\DecMargin{1em}

\subsection{Bayesian Information Criterion difference ratio for VL-Granger causality}
\label{sec:BICDiffratio}

Given $RRSS$ is a restricted residual sum of squares from a regression of $Y$ on $Y$ past, and $T$ is a length of time series, the BIC of null model can be defined below. 
\label{sec:BICdiffratio}
\begin{equation}
	BIC_0(r_{Y})=\frac{RRSS(r_{Y})}{T}T^{(\delta_{max}+1)/T},
	\label{eq:BICrest}
\end{equation}

For unrestricted model, given $URSS$ is an unrestricted residual sum of squares from a regression of $Y$ on $Y,X$ past, and $T$ is a length of time series, the BIC of alternative model can be defined below. 
\begin{equation}
	BIC_1(r_{YX})=\frac{URSS(r_{YX})}{T}T^{(2\delta_{max}+1)/T},
	\label{eq:BICunrest}
\end{equation}
We use the \textit{Bayesian Information Criterion difference ratio} as a main criteria to determine whether $X$ Granger-causes $Y$ or determining $BIC_1(r_{YX}) \ll BIC_0(r_{Y})$ in Algorithm~\ref{algo:VLGrangerCalFunc} line 6, which can be defined below:
\begin{equation}
	\mathrm{r}(BIC_0(r_{Y}),BIC_1(r_{YX}))=\frac{BIC_0(r_{Y})-BIC_1(r_{YX})}{BIC_0(r_{Y})}.
	\label{eq:BICdiffRatio}
\end{equation}

The ratio $\mathrm{r}(\cdot,\cdot)$ is within $[-\infty,1]$. The closer $\mathrm{r}(\cdot,\cdot)$  to $1$, the better the performance of alternative model is compared to the null model. We can set the threshold $\gamma\in[0,1]$ to determine whether $X$ Granger-causes $Y$, i.e. $\mathrm{r}(BIC_0(r_{Y}),BIC_1(r_{YX}))\geq \gamma$ implies $X$ Granger-causes $Y$. Other options of determining $X$ Granger-causes $Y$ is to use F-test or the emulation similarity $simValue$.

\subsection{VL-Transfer-Entropy causality operationalization}
\label{sec:opVLTE}
Given time series $X,Y$, and the maximum possible lag $\delta_{max}$, and whether we want to check for variable or fixed lag $FixLag$, Algorithm~\ref{algo:VLTransferEFunc} reports whether $X$ causes $Y$ by setting $TransEResult$ to be true or false, and by reporting on two transfer entropy values: $\mathcal{T}_{X\xrightarrow{} Y}$ and $\mathcal{T}_{Y\xrightarrow{} X}$.

First, if $FixLag$ is true, then we compute the transfer entropy (line 1) using RTransferEntropy($X,Y$)~\cite{BEHRENDT2019100265}. If $FixLag$ is false, then, we reconstructed $X^{DTW}$ using $DTWReconstructionFunction(X,Y)$ in Section~\ref{sec:DTWRecont} (line 2). We compute the VL-transfer entropy (line 3) using RTransferEntropy($X^{DTW},Y$).

If the ratio $\mathcal{T}(X,Y)_\textrm{ratio}>1$ (Section~\ref{sec:VLTEmethod}), then $X$ causes $Y$ and we set $TransEResult=true$ (line 5), otherwise, $TransEResult=false$ (line 6). 

\setlength{\intextsep}{0pt}
\IncMargin{1em}
\begin{algorithm2e}
\caption{VLTransferEFunc}
\label{algo:VLTransferEFunc}
\SetKwInOut{Input}{input}\SetKwInOut{Output}{output}
\Input{$X,Y$, $\delta_{max}$, $FixLag$ }
\Output{ $TransEResult$,$\mathcal{T}_{X\xrightarrow{} Y}$,$\mathcal{T}_{Y\xrightarrow{} X}$}
\begin{small}
\SetAlgoLined
  \uIf{$FixLag$ is true}{
   \nl $\mathcal{T}_{X\xrightarrow{} Y}$,$\mathcal{T}_{Y\xrightarrow{} X}$ = RTransferEntropy($X,Y$)~\cite{BEHRENDT2019100265}\;
  }
 \Else{
 \nl $X^{DTW}$,$simValue$ = DTWReconstructionFunction($X,Y$) \;
 \nl $\mathcal{T}_{X\xrightarrow{} Y}$,$\mathcal{T}_{Y\xrightarrow{} X}$ = RTransferEntropy($X^{DTW},Y$)~\cite{BEHRENDT2019100265}\;
 } 
 
 \nl \uIf{$\mathcal{T}(X,Y)_\textrm{ratio}>1$  } 
{
\nl $TransEResult=true$
}
 \Else{
\nl $TransEResult=false$ \;}
\nl return $TransEResult$,$\mathcal{T}_{X\xrightarrow{} Y}$,$\mathcal{T}_{Y\xrightarrow{} X}$\;
\end{small}
\end{algorithm2e}\DecMargin{1em}

Additionally, the work by Dimpfl and Peter (2013) ~\cite{dimpfl2013using} proposed the approach to perform the Markov block bootstrap on transfer entropy so that the results can be calculated the p-value of significance tests. The approach preserves dependency within time series while performing bootstrapping. We also integrated this option of bootstrapping analysis in our framework. 

\subsection{Transfer Entropy Ratio}
\label{sec:VLTEmethod}
To determine whether $X$ Transfer-Entropy-causes $Y$, we can use the \textit{Transfer Entropy Ratio} below.

\begin{equation}
\label{eq:trnsEnpRatio}
    \mathcal{T}(X,Y)_\textrm{ratio} =\frac{\mathcal{T}_{X\xrightarrow{} Y}}{\mathcal{T}_{Y\xrightarrow{} X}} .
\end{equation}

The \textit{VL-Transfer Entropy Ratio} is defined below:
\begin{equation}
\label{eq:VLtrnsEnpRatio}
    \mathcal{T}^\text{VL}(X,Y)_\textrm{ratio} =\frac{\mathcal{T}^\text{VL}_{X\xrightarrow{} Y}}{\mathcal{T}^\text{VL}_{Y\xrightarrow{} X}} .
\end{equation}
Where $\mathcal{T}^\text{VL}_{X\xrightarrow{} Y}$ and $\mathcal{T}^\text{VL}_{Y\xrightarrow{} X}$ are Transfer Entropy values from VL-Transfer Entropy (Algorithm~\ref{algo:VLTransferEFunc} line 3).
$\mathcal{T}(X,Y)_\textrm{ratio}$ greater than $1$ implies that $X$ causes $Y$ in Transfer Entropy. The higher $\mathcal{T}(X,Y)_\textrm{ratio}$, the higher the strength of $X$ causing $Y$. The same is true for $\mathcal{T}^\text{VL}(X,Y)_\textrm{ratio}$. 
\section{Experiments}
We measured our framework performance on the task of inferring causal relations using both simulated and real-world datasets. The notations and symbols we use in this section are in Table~\ref{tb:symbloTable}.

\subsection{Experimental setup}
\label{sec:ExpSet}

\begin{table}
\caption{Notations and symbols }
\label{tb:symbloTable}
\begin{small}
\begin{tabular}{|l|p{4in}|}
\hline
{\bf Term and notation} & {\bf Description}                                                                                                                                \\ \hline\hline
$T$               & Length of time series.                                                                                                                      \\ \hline
$\gamma$          & Threshold of BIC difference ratio in Section~\ref{sec:BICDiffratio}. \\ \hline

$\delta_{max}$    & Parameter of the maximum length of time delay                                                                                               \\ \hline
BIC               & Bayesian Information Criterion, which is used as a proxy  \newline to compare the residuals of regressions of two time series.                            \\ \hline
$A \prec B$       & $B$ emulates $A$. \\ \hline
$\mathcal{N}$       & Normal distribution.                                                                                             \\ \hline
ARMA or A.       & Auto-Regressive Moving Average model.                                                                                             \\ \hline
VL-G           & Variable-lag Granger causality with BIC difference ratio:\newline $X$ causes $Y$ if  BIC difference ratio $\mathrm{r}(BIC_0(r_{Y}),BIC_1(r_{YX}))\geq \gamma$.            \\ \hline
G                 & Granger causality ~\cite{atukeren2010relationship}                                                                                                                          \\ \hline
CG                & Copula-Granger method ~\cite{liu2012sparse}       \\ \hline
SIC               &  Spectral Independence Criterion method ~\cite{shajarisales2015telling}             \\ \hline
TE               &  Transfer entropy ~\cite{BEHRENDT2019100265}           \\ \hline
VL-TE               &  Variable-lag transfer entropy         \\ \hline
TE (boots)              &  Transfer entropy~\cite{BEHRENDT2019100265} with bootstrapping~\cite{dimpfl2013using}          \\ \hline
VL-TE (boots)                &  Variable-lag transfer entropy  with bootstrapping~\cite{dimpfl2013using}         \\ \hline
\end{tabular}
\end{small}
\end{table}

We tested the performance of our method on synthetic datasets, where we explicitly embedded a variable-lag causal relation, as well as on biological datasets in the context of the application of identifying initiators of collective behavior, and on other two real-world casual datasets. 

We compared our methods, VL-Granger causality (VL-G) and VL-Transfer entropy (VL-TE), with several existing methods: Granger causality with F-test (G)~\cite{atukeren2010relationship}, Copula-Granger method (CG)~\cite{liu2012sparse}, Spectral Independence Criterion method (SIC)~\cite{shajarisales2015telling}, and transfer entropy (TE)~\cite{BEHRENDT2019100265}. 

In this paper, we explore the choice of  $\delta_{max}$ in $\{0.1T,0.2T,0.3T,0.4T\}$ for all methods to analyze the sensitivity of each method, where $T$ is the length of time series, and set $\gamma=0.5$ as default unless explicitly stated otherwise\footnote{In VL-Granger causality, the threshold $\gamma=0.5$ implies that the time series $X$ causes $Y$ if the residuals of perdition by the VL-Granger can be reduced compared against the residuals of the null model (using $Y$ past to predict $Y$) at least half. We set the $\gamma=0.5$ for a pairwise time series X because we know they have either a strong signal of causation or no causation.  }.

\subsection{Datasets}

\subsubsection{Synthetic data: pairwise level}
The main purpose of the synthetic data is to generate settings that explicitly illustrate the difference between the original Granger causality, transfer entropy methods and the proposed variable-lag approaches. We generated pairs of time series for which the fixed-lag causality methods would fail to find a relationship but the variable-lag approach would find the intended relationships. 

We generated a set of synthetic time series of 200 time steps. We generated two sets of pairs of time series $X$ and $Y$.
First, we generated $X$ either by drawing the value of each time step from a standard normal distribution  $\mathcal{N}(0,1)$ with zero mean and a variance at one ($X(t)\sim\mathcal{N}$) (normal model) or by Auto-Regressive Moving Average model (ARMA or A.) with $X(t)= 0.2X(t-1) + \epsilon_X$ where $\epsilon_X\sim\mathcal{N}(0,1)$.  

The first set we generated was of explicitly related pairs of time series $X$ and $Y$, where  $Y$ emulates $X$ with some  time lag $\Delta =5$ ($X\prec Y$). Specifically, $Y(t)= X(t-\Delta) + 0.1\epsilon_Y$ where $\epsilon_Y\sim\mathcal{N}(0,1)$.

One way to ensure lag variability is to ``turn off" the emulation for some time. For example, $Y$ remains constant between 110th and 170th time steps imitating the $X$ at 100th time step. This makes $Y$ a variable-lag follower of $X$. Figure~\ref{fig:FixVsVarLagExample} shows examples of the generated time series that has $Y$ remains constant for a while. We generated time series for each generator model 15 times.

The second set of time series pairs $X$ and $Y$ were generated independently and as a result have no causal relation. We used these pairs to ensure that our method does not infer spurious relations.  We generated time series for each generator model 15 times. 

Hence, we have 15 datasets of normal model with $X\prec Y$, 15 datasets of normal model with $X\nprec Y$, 15 datasets of ARMA model with $X\prec Y$, 15 datasets of ARMA model with $X\nprec Y$, and 15 datasets where $X$ is from normal model, and $Y$ is from ARMA model s.t.  $X\nprec Y$. In total, we have 75 datasets for the pairwise-level simulation.  See Appendix~\ref{apdx:SimCode} for the code we used to generated the datasets.

We set the significance level for both F-test and bootstrapping test of transfer entropy at $\alpha=0.05$. For the bootstrapping of transfer entropy, we set the number of bootstrap replicates as 100 times. We  considered there to be a causal relation only if $\mathrm{r}(BIC_0(r_{Y}),BIC_1(r_{YX}))\geq \gamma$ for our method.

For the task of causal prediction, we define the true positive (TP) when  the ground truth is $X\prec Y$ and a method reports that $X\prec Y$. The true negative (TN) is when both the ground truth and predicted result agree that $X\nprec Y$. The false positive (FP) is when the ground truth is $X\nprec Y$, but the method predicted that $X\prec Y$. The false negative (FN) is the ground truth is $X\prec Y$, but the method disagrees.  The accuracy is the TP and TN cases divided by the number of total pairs of time series. The true positive rate (TPR) is the number of TP cases divided by the number of TP and FN cases. The false positive rate (FPR) is the number of FP cases divided by the number of FP and TN cases. 

We report the result in the form of the receiver operating characteristic (ROC) curves. The results of methods are compared against each other using their area under a curve (AUC).  

\subsubsection{Synthetic data: group level} This experiment explores the ability of causal inference methods to retrieve {\em multiple}  causes of a time series $Y_{ij}$, which is generated from multiple time series $X_i,X_j$. Fig.~\ref{fig:SynHSiblingRes} shows the ground truth causal graph we used to generate simulated datasets. The edges represent causal directions from the cause time series (e.g. $X_1$) to the effect time series (e.g. $Y_1$). $Y_{ij}$ represents the time series generated by $agg(\{X'_i,X'_j\})$, where $X_i\prec X'_i$ and $X_j\prec X'_j$ with some fixed lag $\Delta = 5$. The task is to infer edges of this causal graph from the time series. We generated time series for each generator model 15 times. We set $\gamma=0.3$ in this experiment due to the weak signal of $X$ causes $Y$  when there are multiple causes of $Y$. There are also two generators for $X_1,X_2,X_3$: normal distribution and ARMA model.

For the task of causal graph prediction, a TP case is a case when both when both the ground truth and predicted result agree that there is a causal edge from $X_i$ to $Y_j$ in the graph. A TN case is a case when both when both the ground truth and predicted result agree that there is no causal edge from $X_i$ to $Y_j$ in the graph. A FP is a case when there is no edge in a ground truth casual graph, but a method predicted that there is the edge. A FN is a case when there is an edge from $X_i$ to $Y_j$ in a ground truth casual graph, but a method predicted that there is no edge from $X_i$ to $Y_j$. We report precision, recall, and F1 score for all methods. The precision ($prec$) is a ratio between a number of TP cases and a number of TP+FP cases. The recall ($rec$) is a ratio between a number of TP cases and a number of TP+FN cases. The F1 score $F=2*prec*rec/(prec+rec)$.

For the parameter setting, since the the time delay between causes and effects is 5 time steps for all datasets in this section, methods with the $\Delta_{max}$ parameter have $\Delta_{max}=10$.

\begin{figure}
\centering
\includegraphics[width=.3\columnwidth]{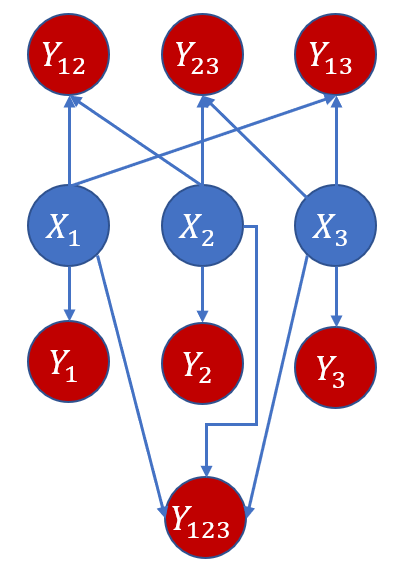}
\caption{ The causal graph where the edges represent causal directions from the cause time series (e.g. $X_1$) to the effect time series (e.g. $Y_1$). $Y_{ij}$ represents a time series generated by $agg(\{X'_i,X'_j\})$, where $X_i\prec X'_i$ with some fixed lag $\Delta$. }
\label{fig:SynHSiblingRes}
\end{figure}

\begin{figure}
\centering
\includegraphics[width=1\columnwidth]{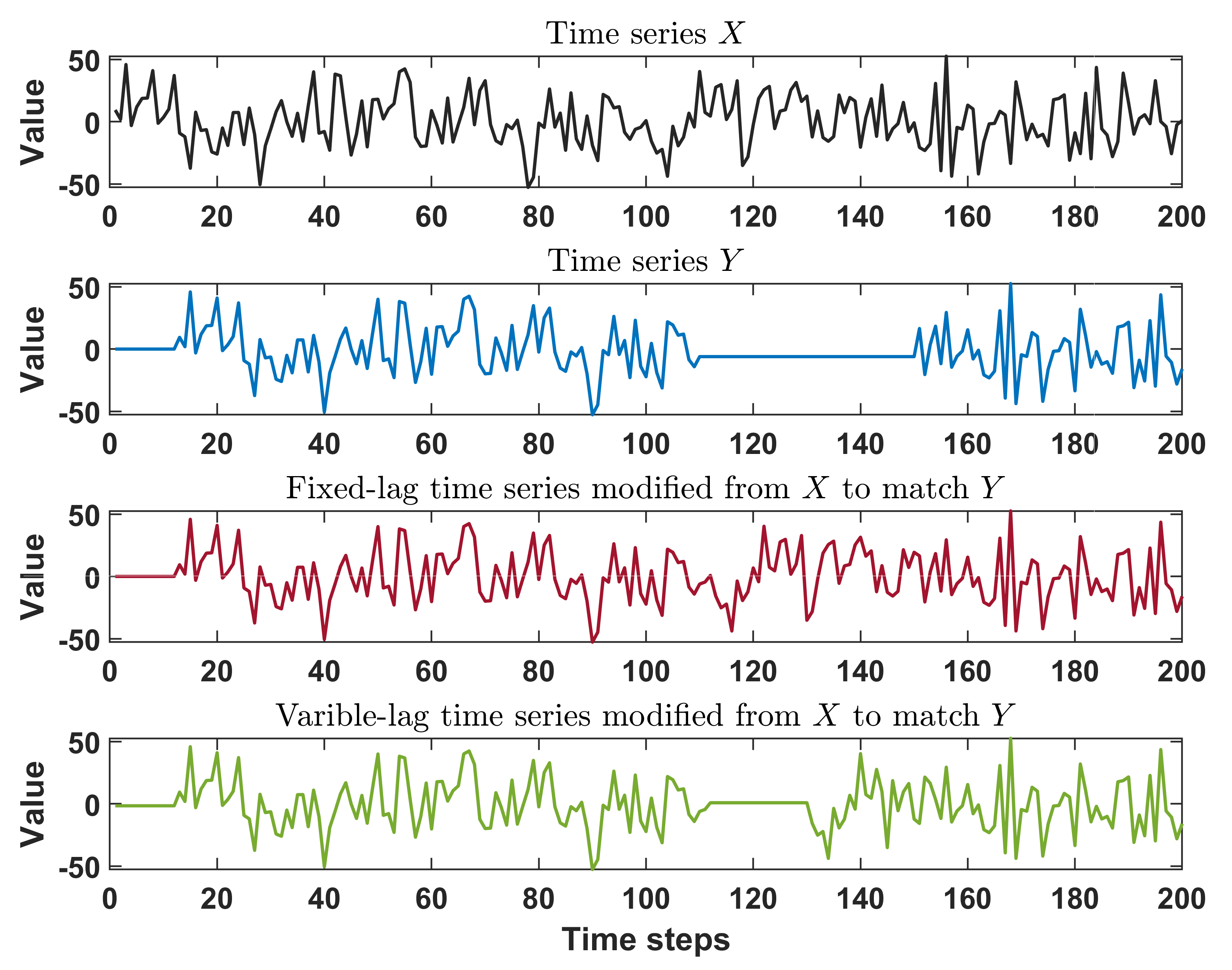}
\caption{ The comparison between the original time series $X$, variable-lag follower $Y$, fixed-lag time series modified from $X$ to match $Y$, and variable-lag time series modified from $X$ to match $Y$. The traditional Granger causality uses only fixed-lag version of $X$ to infer whether $X$ causes $Y$, while our approach uses both versions of $X$ to determine the causality between $X,Y$. Both $X,Y$ are generated from $\mathcal{N}$. $Y$ remains constant from time 110 to 170, which makes it a variable-lag follower of $X$. }
\label{fig:FixVsVarLagExample}
\end{figure}

See Appendix~\ref{apdx:SimCode} for the code we used to generated the datasets.
\subsubsection{Schools of fish}
\begin{figure}
\centering
\includegraphics[width=.8\columnwidth]{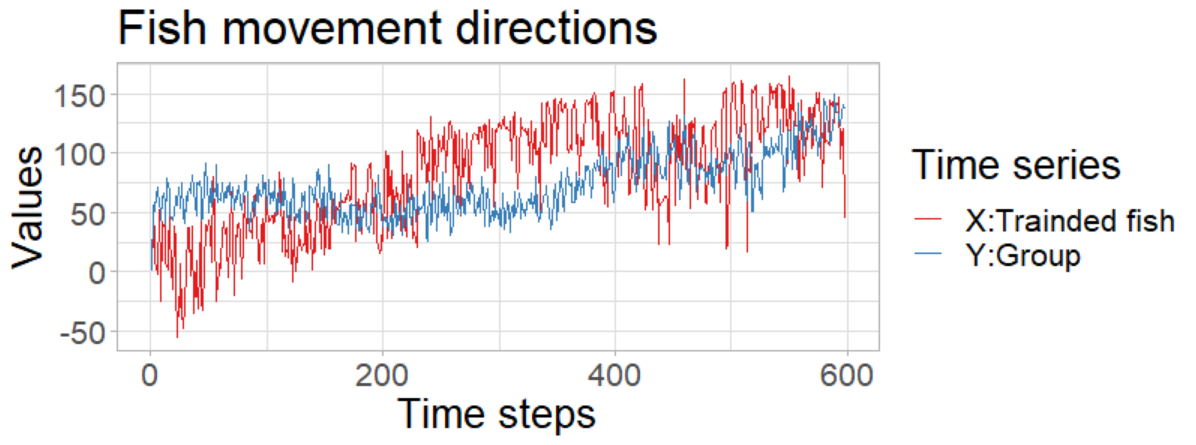}
\caption{ Time series of fish movement: $X$ is an aggregated time series of movement directions of trained fish and $Y$ is an aggregated time series of movement directions of untrained fish, which is the rest of the group. }
\label{fig:fishTS}
\end{figure}

We used the dataset of golden shiners (\emph{Notemigonus crysoleucas}) that is publiclly available. The dataset has been collected for the study of information propagation over the visual fields of fish \cite{strandburg2013visual}.  A coordination event consists of two-dimensional time series of fish movement that are recorded by video. The time series of fish movement are around 600 time steps.  The number of fish in each dataset is around 70 individuals, of which 10 individuals are ``informed" fish who have been trained to go to a feeding site. Trained fish lead the group to feeding sites while the rest of the fish just follow the group.  We represent the dataset as a pair of aggregated time series: $X$ being the aggregated time series of the directions of trained fish and $Y$ being the aggregated time series of the directions of untrained fish (see Fig.~\ref{fig:fishTS}). The task is to infer whether $X$ (trained fish) is  a cause of $Y$ (the rest of the group).

\subsubsection{Troop of baboons} 
\begin{figure}
\centering
\includegraphics[width=.8\columnwidth]{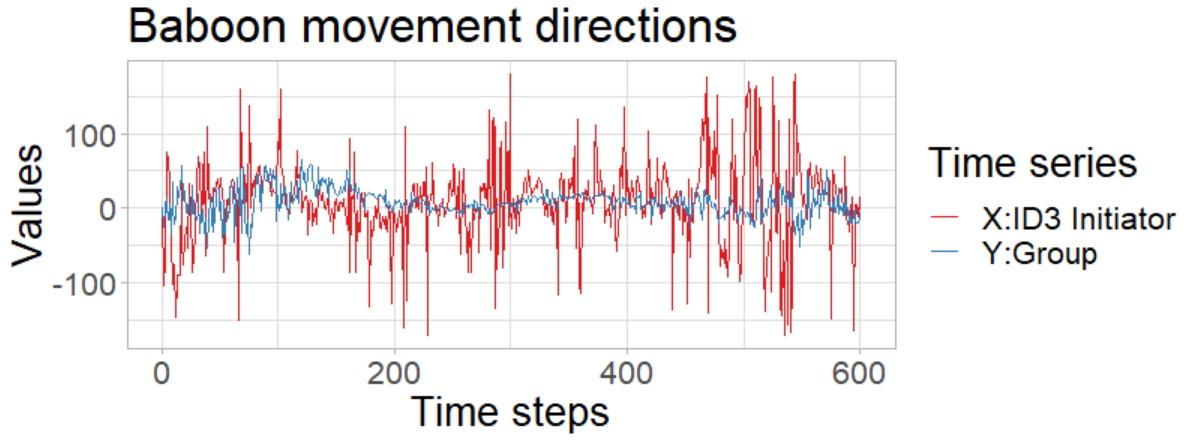}
\caption{ Time series of baboon movement: $X$ is a time series of movement directions of ID3 and $Y$ is an aggregated time series of movement directions of the rest of the group. }
\label{fig:baboonTS}
\end{figure}

We used another publicly available dataset of animal behavior, the movement of a troop of olive baboons (\emph{Papio anubis}). The dataset consists of GPS tracking information from 26 members of a troop, recorded at 1 Hz from 6 AM to 6 PM between August 01, 2012 and August 10, 2012. 
The troop lives in the wild at the Mpala Research Centre, Kenya \cite{crofoot2015data,strandburg2015shared}. For the analysis, we selected the 16 members of the troop that have GPS information  available for 10 consecutive days, with no missing data. We selected  a set of trajectories of lat-long coordinates from a highly coordinated event that has the length of 600 time steps (seconds) for each baboon.  This known coordination event is on August 02, 2012 in the morning, with the baboon ID3 initiating the movement, followed by the rest of the troop~\cite{FLICAtkdd}. Again, the goal is to infer ID3 (time series $X$) as the cause of the movement of the rest of the group (aggregate time series $Y$) (see Fig.~\ref{fig:baboonTS}).

\subsubsection{Gas furnace}
\begin{figure}
\centering
\includegraphics[width=.75\columnwidth]{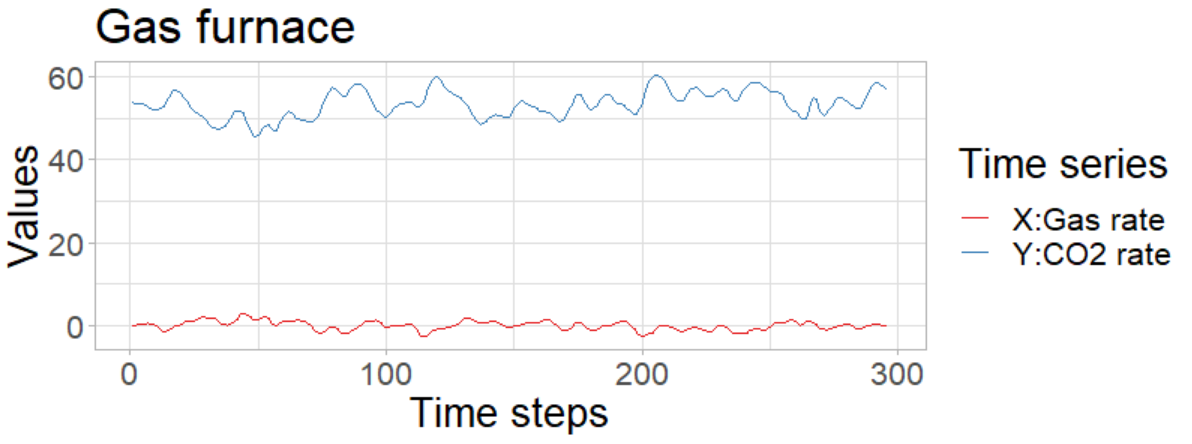}
\caption{ Time series of Gas furnace: $X$ is time series of gas consumption rate and $Y$ is time series of $CO_2$. }
\label{fig:gasfurnaceTS}
\end{figure}

 This dataset consists of information regarding a gas consumption by a gas furnace~\cite{box2015time}. $X$ is time series of gas consumption rate and $Y$ is time series of $CO_2$ rate produced by a gas furnace (see Fig.~\ref{fig:gasfurnaceTS}). Both $X,Y$ have 296 time steps.

\subsubsection{Old Faithful geyser eruption}

\begin{figure}
\centering
\includegraphics[width=.9\columnwidth]{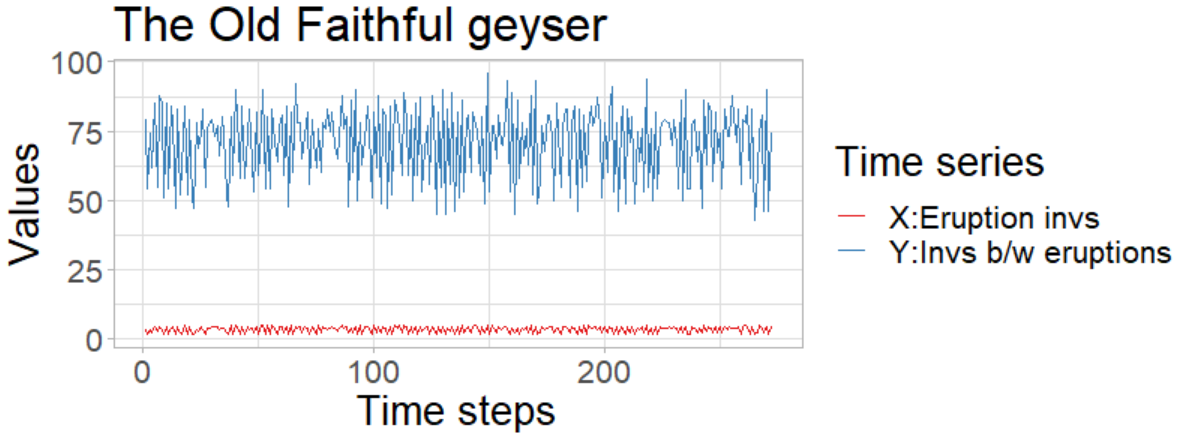}
\caption{ Time series of the Old Faithful geyser eruption: $X$ is time series of eruption intervals and $Y$ is time series of the intervals between current eruption and the next eruption. }
\label{fig:OldFFGeyserTS}
\end{figure}

 This dataset consists of information regarding  eruption durations and intervals between eruption events at Old Faithful geyser~\cite{azzalini1990look}. $X$ is time series of eruption duration and $Y$ is time series of the interval between current eruption and the next eruption (see Fig.~\ref{fig:OldFFGeyserTS}). Both $X,Y$ have 298 time steps.

\subsection{Time complexity and running time}

\begin{table}
\caption{Running time of our approaches with varying time series length $T$ and maximum time delay $\delta_{max}$.}
\label{tab:runningTime}
\begin{small}
\begin{tabular}{c|c|c|c|c|}
 \cline{2-5} 
                                                                      & \multicolumn{2}{c|}{\cellcolor[HTML]{CBCEFB}VL-G} & \multicolumn{2}{c|}{\cellcolor[HTML]{CBCEFB}VL-TE} \\ \hline
\rowcolor[HTML]{C0C0C0} 
\multicolumn{1}{|c|}{\cellcolor[HTML]{C0C0C0}$\delta_{max}/T$} & $T=5000$                  & $T=20000$                 & $T=5000$                  & $T=20000 $                 \\ \hline
\multicolumn{1}{|c|}{0.05}                                            & 5.39                    & 110.00                  & 17.57                   & 126.02                   \\ \hline
\rowcolor[HTML]{EFEFEF} 
\multicolumn{1}{|c|}{\cellcolor[HTML]{EFEFEF}0.10}                    & 7.90                    & 128.19                  & 17.42                   & 121.38                   \\ \hline
\multicolumn{1}{|c|}{0.20}                                            & 9.22                    & 200.17                  & 17.93                   & 131.23                   \\ \hline
\end{tabular}
\end{small}
\end{table}

The main cost of computation in our approach is DTW. We used the ``Windowing technique'' for the search area of warping ~\cite{keogh2001derivative}. The main parameter for windowing technique is the maximum time delay $\delta_{max}$. Hence, the time complexity of VL-G is $\mathcal{O}(T\delta_{max})$.  The time complexity of TE can be at most $\mathcal{O}(T^3)$~\cite{shao2014accelerating}, which makes VL-TE has the same time complexity. However, with the work by Kontoyiannis and Skoularidou in~\cite{kontoyiannis2016estimating}, the convergence rate of TE approximation can be reduced to $\mathcal{O}(1/\sqrt{T})$ if time series are generated with a Markov-chain property of a given lags.
Table~\ref{tab:runningTime} shows the running time of our approach on time series with the varying length ($T\in \{5000,20000\}$) and maximum time delay ($\delta_{max} \in \{0.05T,0.1T,0,2T\}$).

\section{Results}
We report the results of our proposed approaches and  other methods on both synthetic and real-world datasets. We also explore how the performance of the methods depends on the basic parameter, $\delta_{\max}$.

\subsection{Synthetic data: pairwise level}
\label{sec:simResPairwise}

\begin{figure}
\centering
\includegraphics[width=.8\columnwidth]{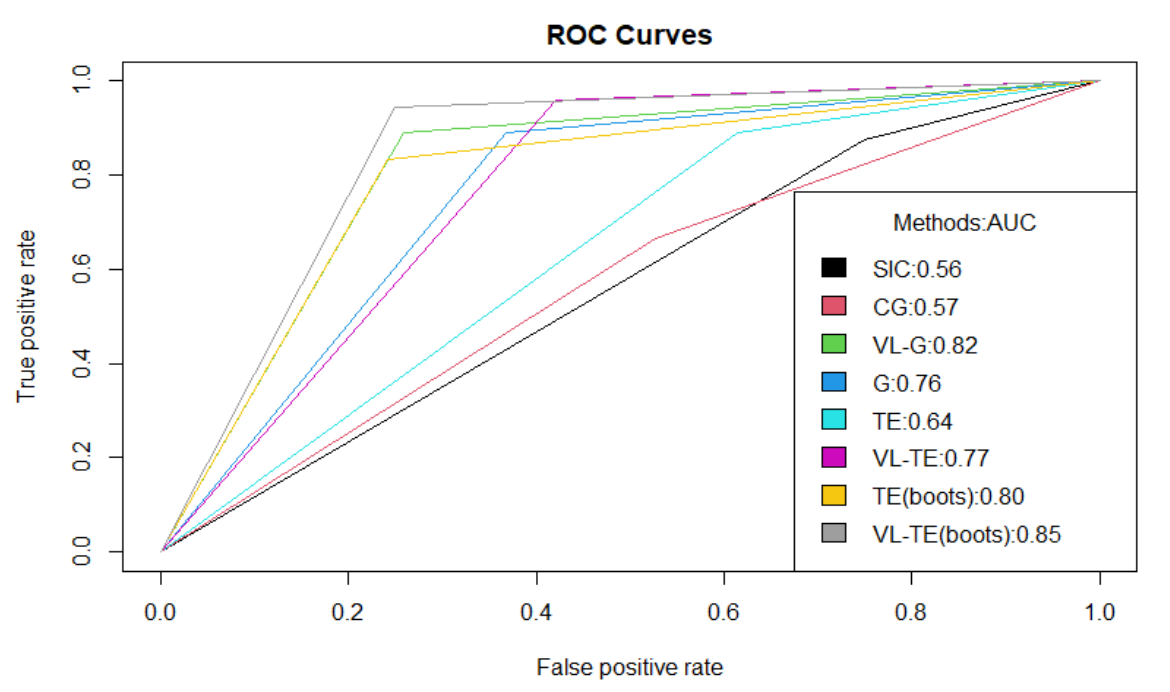}
\caption{ The ROC curves from the results of prediction using pairwise-level simulation datasets.}
\label{fig:ROCplot}
\end{figure}

Figure~\ref{fig:ROCplot} shows the ROC curves from the results of inferring causal relations and directions. According to the AUC values, all variable-lag methods perform better than their original methods (e.g. VL-G vs. G, VL-TE vs. TE).

 The result also shows that our method, VL-Transfer entropy with bootstrapping, VL-TE (boots),  performed better than the rest of other methods. The second best method is VL-Granger causality (VL-G), which has the AUC value almost the same as VL-TE (boots). For transfer entropy results, the bootstrapping methods (both VL-TE (boots) and TE (boots) ), performed better than their original version. This indicates that the bootstrapping approach increases the performance of transfer entropy methods in this task. 

Moreover, we also investigated the sensitivity of varying the value of the $\delta_{max}$ parameter for all methods. We  aggregated the accuracy of inferring causal direction from various cases that have the same $\delta_{max}$ value and report the result. The result in Fig.~\ref{fig:SynAvgAccDeltaMax} shows that VL-TE (boots), VL-G, TE (boots), and G can maintain the high accuracy (>0.9) throughout the range of the  values of $\delta_{max}$.  

\begin{figure}
\centering
\includegraphics[width=.8\columnwidth]{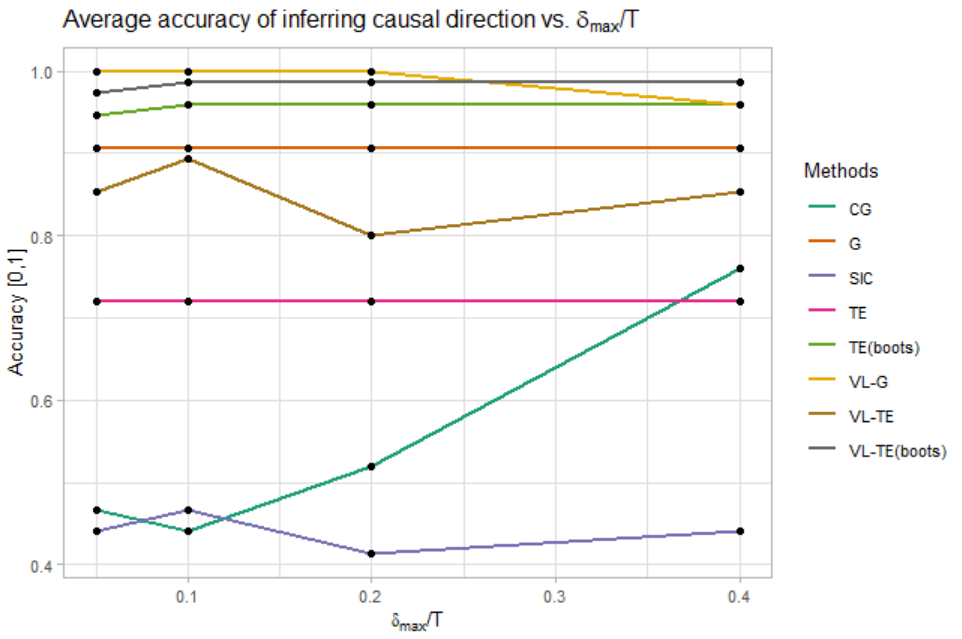}
\caption{ Average accuracy of inferring causal direction as a function of $\delta_{max}$. $x$-axis represents the value of $\delta_{max}$ as a fraction of the time series length $T$ and $y$-axis is the average accuracy. }
\label{fig:SynAvgAccDeltaMax}
\end{figure}

\subsection{Synthetic data: group level}
\label{sec:simResGroup}

 Table~\ref{tb:GrHSiblingSynRes} shows the result of causal graph inference. The VL-G performed the best overall with the highest F1 score.  This result reflects the fact that our approaches can handle complicated time series in causal inference task better than the rest of other methods. VL-TE also performed better than TE.  
 
 In addition, we aggregated $X=agg(\{X_1,X_2,X_3\})$ and $Y=agg(\{Y_1,Y_2,\dots,Y_{123}\})$, then we measured the ability of methods to infer that $X$ is  a cause of $Y$. The results, which are in the ``Group: $X\prec Y$" column in Table~\ref{tb:GrHSiblingSynRes}, show that   G, CG and SIC performed well in this task, while the rest of methods failed to infer causal relations. Note that VL-G also performed well when we relaxed the $\gamma$ from 0.3 to 0.01. This is due to the fact that the aggregated group time series have a complicated casual relation between  $X=agg(\{X_1,X_2,X_3\})$ and $Y=agg(\{Y_1,Y_2,\dots,Y_{123}\})$, which implies that the causal signal is not strong. Hence, we need to relaxed the $\gamma$ to capture the causal relation.
 
 Comparing transfer entropy methods, the bootstrapping approach decreaseed the performance to detect causal relations compared to their original version. This is also due to the weak signal of causal relation in the complicated datasets. 
 
 Overall, the simple original Granger causality performed well in both tasks. Moreover, due to the causal relations in simulation datasets are highly linear, hence, we expect the linear model (e.g. VL-G, G) should perform better than the non-linear approaches (e.g. TE, VL-TE). 
 
\begin{table}[]
\caption{The results of the precision, recall, and F1-score values of edges inference of causal graph in Fig.~\ref{fig:SynHSiblingRes}.  Each row is a method and each column is a measure type. The * indicates that the parameter $\gamma$ is changed from 0.3 to 0.01}
\label{tb:GrHSiblingSynRes}
\begin{small}
\begin{tabular}{c|c|c|c|c|}
\cline{2-5}
                                                      & \multicolumn{3}{c|}{\cellcolor[HTML]{CBCEFB}Causal graph} & \cellcolor[HTML]{CBCEFB}Group: $X\prec Y$ \\ \hline
\rowcolor[HTML]{C0C0C0} 
\multicolumn{1}{|c|}{\cellcolor[HTML]{C0C0C0}Methods} & Precision           & Recall          & F1 score          & Accuracy                                             \\ \hline
\multicolumn{1}{|c|}{VL-G}                            & 0.93	& 0.83	 & 0.87
              & 0.23/0.93*                                              \\ \hline
\rowcolor[HTML]{EFEFEF} 
\multicolumn{1}{|c|}{\cellcolor[HTML]{EFEFEF}G}       & 0.71	&0.99	&0.83
              & 0.97                                                 \\ \hline
\multicolumn{1}{|c|}{CG}                              & 0.04	&0.12	&0.06
              & 0.90                                                 \\ \hline
\rowcolor[HTML]{EFEFEF} 
\multicolumn{1}{|c|}{\cellcolor[HTML]{EFEFEF}SIC}     &0.03	 &0.11	& 0.05
              & 0.93                                                 \\ \hline
\multicolumn{1}{|c|}{TE}                              & 0.17	& 0.62	& 0.26
              & 0.50                                                 \\ \hline
\rowcolor[HTML]{EFEFEF} 
\multicolumn{1}{|c|}{\cellcolor[HTML]{EFEFEF}VL-TE}   &0.24	& 0.71	& 0.35
              & 0.47                                                 \\ \hline
\multicolumn{1}{|c|}{TE (boots)}                 & 0.08	& 0.17	&0.11
              & 0.30                                                 \\ \hline
\rowcolor[HTML]{EFEFEF} 
\multicolumn{1}{|c|}{\cellcolor[HTML]{EFEFEF}VL-TE (boots)}   & 0.08	&0.18	&0.11
      & 0.07                                                 \\ \hline
\end{tabular}
\end{small}
\end{table}

\subsection{Real-world datasets}

\begin{table}[]
\caption{The result of inferring causal relations in real-world datasets. Each row is a dataset and each column is a method. An element is one if a method successfully inferred a causal relation with some parameter, while an element is zero if no parameter setting in a method can be used to successfully inferred a causal relation. For VL-G, we used both  BIC difference ratio and F-test to infer causal relation. The * implies that VL-G with BIC difference ratio failed to detect causal relations but VL-G with F-test successfully detect the relations. For fish and baboon datasets, VL-G with both criteria were able to detect causal relations. }
\label{tab:RealWorldResult}
\begin{small}
\begin{tabular}{c|c|c|c|c|c|c|c|c|}
\cline{2-9}
                                                      & \multicolumn{8}{c|}{\cellcolor[HTML]{CBCEFB}Methods} \\ \hline
\rowcolor[HTML]{C0C0C0} 
\multicolumn{1}{|c|}{Case} & VL-G & G & CG & SIC & TE & VL-TE & TE (boots) & VL-TE (boots) \\ \hline
\multicolumn{1}{|c|}{Fish}                         & 1    & 0 & 1  & 0   & 1  & 1 & 0  & 0    \\ \hline
\rowcolor[HTML]{EFEFEF} 
\multicolumn{1}{|c|}{Baboon}                       & 1    & 1 & 1  & 1   & 1  & 1  & 0  & 0    \\ \hline
\multicolumn{1}{|c|}{Gas furnace}                  & 1*   & 1 & 0  & 1   & 1  & 1   & 1  & 0   \\ \hline
\rowcolor[HTML]{EFEFEF} 
\multicolumn{1}{|c|}{Old faithful geyser}          & 1*   & 0 & 1  & 1   & 0  & 1   & 0  & 0   \\ \hline
\end{tabular}
\end{small}
\end{table}

Table~\ref{tab:RealWorldResult} shows results of inferring causal relations in real-world datasets. For  VL-G, it performed better than G. However, BIC difference ratio failed to infer causal relations of  gas furnace and old faithful geyser datasets but F-test successfully inferred causal relations in all datasets. Typically,  a  causal relation that has a high BIC difference ratio can also be detected to have a causal relation by F-test but not vise versa. This suggests that gas furnace and old faithful geyser have weak causal relations. For G, the method cannot detect fish and Old faithful geyser datasets. This suggests that both datasets have a high-level of variable lags that a fixed-lag assumption in G has an issue. For CG, SIC, and TE, they failed in one dataset each. This implies that some dataset that a specific approach failed to detect a causal relation has broke some assumption of a specific approach. Lastly, VL-TE was able to detect all causal relations. 

For the old faithful geyser dataset, both G and TE failed to detect a causal relation while both VL-G and VL-TE successfully inferred a causal relation. This implies that this dataset has a high-level of variable lags that broke a fix-lag assumption of G and TE.  

Lastly, the transfer entropy methods with bootstrapping almost failed to detect anything. This is due to the weak signal of causal relations in real-world datasets.

\subsection{Variable lags vs. fixed lag}

\subsubsection{VL-Granger causality} To compare the performance of VL-G and G, we simulated 100 datasets of $X \prec Y$ with variable lags. Since $X \prec Y$, a higher BIC difference ratio implies a better result. Fig.~\ref{fig:VLGrangerBICDiffRatio} shows the results of BIC difference ratio for VL-G and G. Obviously, VL-G has a higher BIC difference ratio than G's. This suggests that VL-G was able to capture stronger signal of $X$ causes $Y$. 

\begin{figure}
\centering
\includegraphics[width=1\columnwidth]{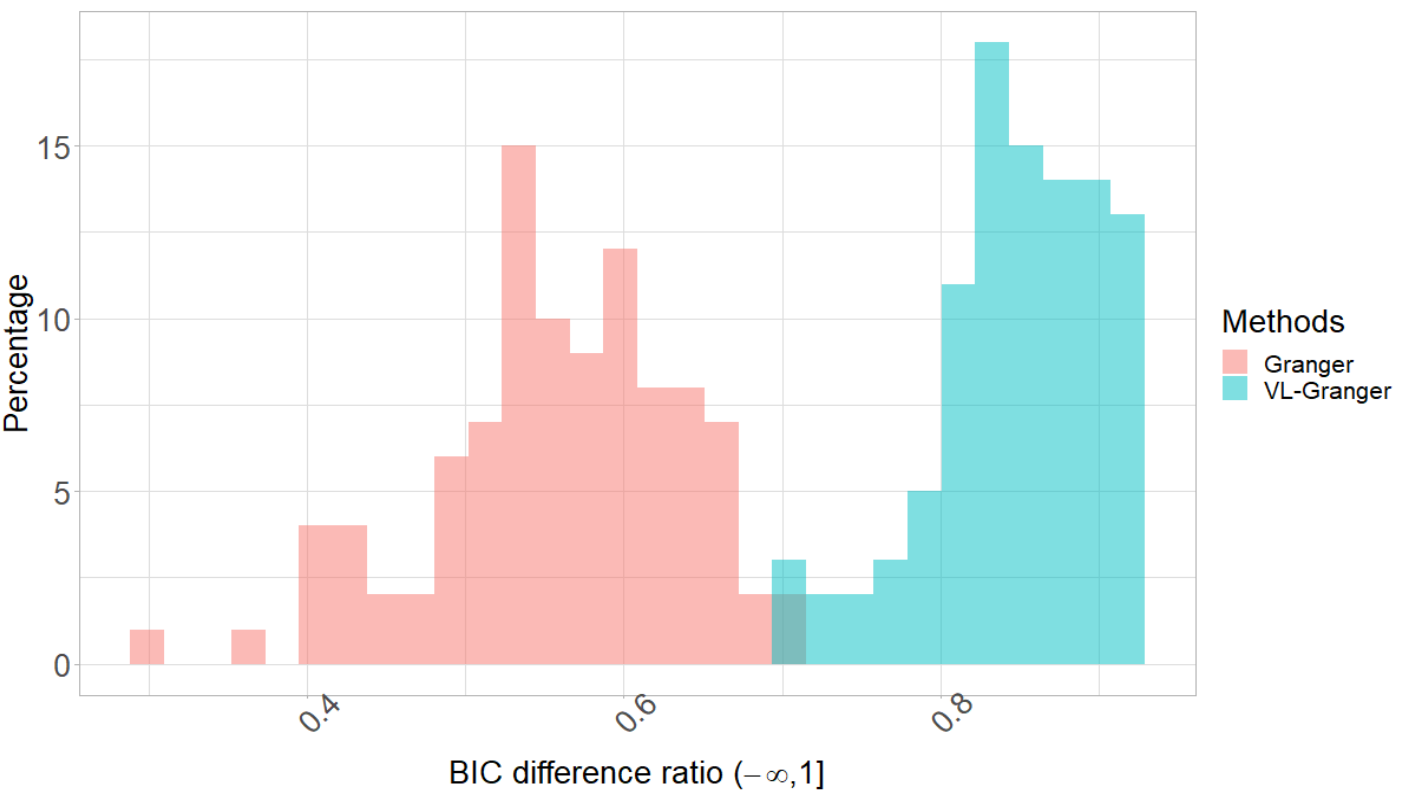}
\caption{ Empirical distributions of BIC difference ratios of VL-Granger and Granger methods inferred from simulation data of $X \prec Y$. Higher BIC difference ratio implies better model if $X$ is the cause of $Y$.  }
\label{fig:VLGrangerBICDiffRatio}
\end{figure}

\subsubsection{VL-Transfer Entropy} To compare the performance of VL-TE and TE, we also simulated 100 datasets of $X \prec Y$ with variable lags. Since $X \prec Y$, a higher transfer entropy ratio implies a better result. Fig.~\ref{fig:VL-TERatio} shows the results of transfer entropy ratio for VL-TE and TE. Obviously, VL-TE has a higher transfer entropy ratio than TE's. This suggests that VL-TE was able to capture stronger signal of $X$ causes $Y$. 

\begin{figure}
\centering
\includegraphics[width=1\columnwidth]{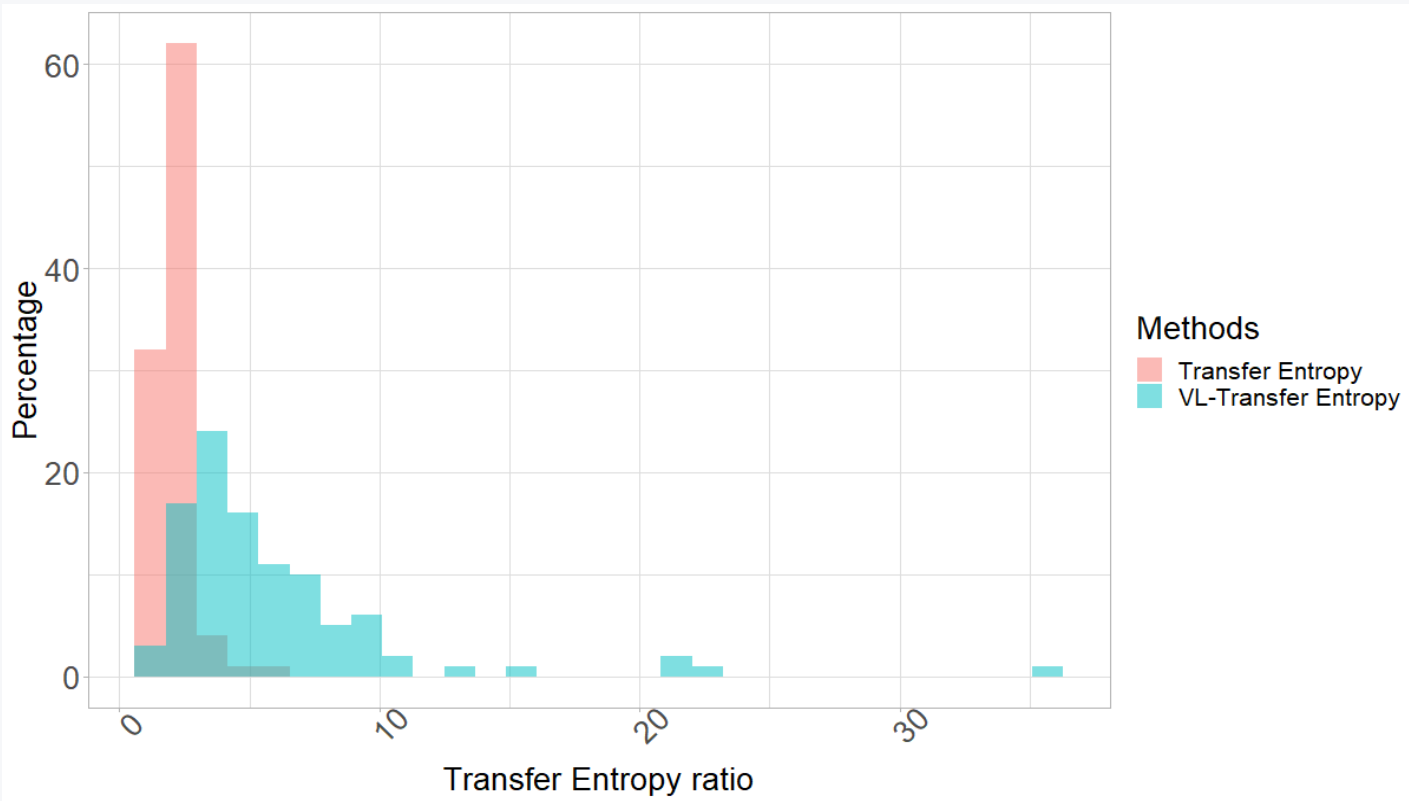}
\caption{  Empirical distributions of transfer entropy ratios of VL-transfer entropy and transfer entropy methods inferred from simulation data of $X \prec Y$. Higher transfer entropy ratio implies better model if $X$ is the cause of $Y$.   }
\label{fig:VL-TERatio}
\end{figure}

\section{Conclusions}
In this work, we proposed a method to infer Granger and transfer entropy causal relations  in time series where the causes influence effects with arbitrary time delays, which can change dynamically. We formalized a new Granger causal relation and a new transfer entropy causal relation, proving that they are true generalizations of the traditional Granger causality and transfer entropy respectively. We demonstrated on both carefully designed synthetic datasets and noisy real-world datasets that the new causal relations can address the arbitrary-time-lag influence between cause and effect, while the traditional Granger causality and transfer entropy cannot. Moreover, in addition to  improving and extending Granger causality and transfer entropy,  our approach can be applied to infer leader-follower relations, as well as the dependency property between cause and effect. Note that, in simulation datasets, we did not include nonlinear datasets in our analysis. We expect that the linear measures (e.g. VL-Granger and Granger) should outperform the non-linear measures (Transfer Entropy and VL-Transfer Entropy) in the linear datasets, while the non-linear measures should outperform linear measures in non-linear datasets. 

We have shown that, in many situations, the causal relations between time series do not have a lock-step connection of a fixed lag that the traditional Granger causality and transfer entropy assume. Hence, traditional Granger causality and transfer entropy missed true existing causal relations in such cases, while our methods correctly inferred them. Our approach can be applied in any domain of study where the causal relations between time series is of interest. The R-CRAN package entitled \textit{VLTimeCausality} is provided at~\cite{ShareSourcecode}. See Appendix~\ref{apdx:manual} for the example of how to use the package. 
\appendix
\section{Appendix: Dynamic time warping}
\label{apdx:DTW}
The Dynamic Time Warping (DTW)~\cite{sakoe1978dynamic} is one of well-known distance measures between a pairwise of time series. The main idea of DTW is to compute the distance from the matching of similar elements between time series.  The series of indices of matching is called ``Warping path". Given time series $X,Y$ that have length $T_X$ and $T_Y$ respectively, their warping path is defined as $P=(\Delta_1,\dots,\Delta_K)$ where the following conditions are true~\cite{keogh2001derivative}: 

\numsquishlist
\item $\Delta_1 = (1,1)$,
\item $\Delta_K = (T_X,T_Y)$,
\item $\max(\{T_X,T_Y\})\leq K < T_X+T_Y-1$ and
\item for all pair $\Delta_{t-1} = (i',j'),\Delta_{t}= (i,j)$, we have $\Delta_{t-1}\in \{(i-1,j),(i,j-1),(i-1,j-1)\}$ where  $i'\geq 1$ and $j'\geq 1$.
\numsquishend

Each $\Delta = (i,j)$ in $P$ represents the matching indices where $X(i)$ is matched with $Y(j)$.  Suppose $\mathbb{P}$ is a set of all possible warping paths that satisfy the conditions above, the following equation represents the DTW distance between $X,Y$. 

\begin{equation}
\label{eq:dtwdist}
    d_{DTW} = \min_{P \in \mathbb{P}} \sum_{\Delta_t \in P,\Delta_t = (i,j)} D(i,j).
\end{equation}
Where $D(i,j)$ is a distance function between $X(i),Y(j)$. If we use the Euclidean distance, then $D(i,j)=\sqrt{X(i)^2+Y(j)^2}$. A warping path $P^*$ that minimizes the Eq.~\ref{eq:dtwdist} is called an ``optimal warping path". The Eq.~\ref{eq:dtwdist} solution can be solved by the dynamic programming technique.  In the the dynamic programming, given $\mathcal{D}(i,j)$ as a DTW distance of time series $X$ within the interval $[1,i]$, and time series $Y$ within the interval $[1,j]$, we can use the following equation to compute  $\mathcal{D}(i,j)$~\cite{senin2008dynamic}.

\begin{equation}
  \mathcal{D}(i,j)=\left\{
  \begin{array}{@{}ll@{}}
	D(i,j), & i =1,j=1 \\
   \mathcal{D}(i,j-1)+D(i,j), & i=1,j>1 \\
    \mathcal{D}(i-1,j)+D(i,j), & i>1,j=1\\
    D(i,i)+\min( \{\mathcal{D}(i-1,j),\mathcal{D}(i,j-1),\mathcal{D}(i-1,i-1)\}     ), & \mathrm{Otherwise}.
  \end{array}\right.
\label{eq:DTWdyproc}
\end{equation}

For time series $X,Y$, our goal is to compute the DTW distance $d_{DTW}=\mathcal{D}(T_X,T_Y)$, of which its solution can be founded using the Algorithm~\ref{algo:DTWDistFunc}.

\setlength{\intextsep}{0pt}
\IncMargin{1em}
\begin{algorithm2e}
\caption{ DTWFunction}
\label{algo:DTWDistFunc}
\SetKwInOut{Input}{input}\SetKwInOut{Output}{output}
\Input{ Time series $X,Y$ that have length $T_X$ and $T_Y$ respectively.}
\Output{ $T_X\times T_Y$-Matrix $\mathcal{D}$, the DTW distance $d_{DTW}$, and DTW optimal warping path $P$. }
\begin{small}
\SetAlgoLined
\nl Let $D$ be a $T_X\times T_Y$-Matrix matrix of  Euclidean distances of elements $X$ and $Y$ s.t. $D(i,j)=\sqrt{X(i)^2+Y(j)^2}$\;
\nl Set $\mathcal{D}(1,1)=D(1,1)$\;
\nl \For{$t=2\to T_X$}{
$\mathcal{D}(t,1)=\mathcal{D}(t-1,1)+D(t,1)$\;
}
\nl \For{$t=2\to T_Y$}{
$\mathcal{D}(1,t)=\mathcal{D}(1,t-1)+D(1,t)$\;
}

\nl \For{$t_X=2\to T_X$}
{
\nl    \For{$t_Y=2\to T_Y$}{
\nl    $\mathcal{D}(t_X,t_Y)= D(t_X,t_Y)+\min( \{\mathcal{D}(t_X-1,t_Y),\mathcal{D}(t_X,t_Y-1),\mathcal{D}(t_X-1,t_Y-1)\}     )$\;
    }
}
\nl $d_{DTW}=\mathcal{D}(T_X,T_Y)$\;
\nl $P^*$=WarpingPathFindingFunction($\mathcal{D}$)\;
\nl Return $\mathcal{D},d_{DTW},P^*$\;
\end{small}
\end{algorithm2e}\DecMargin{1em}

In Algorithm~\ref{algo:DTWDistFunc} line 1, we compute Euclidean distance for all pair $X(i),Y(j)$ and keep the result in $D(i,j)$. Then, in the line 2-4, we compute the base-case distance ($\mathcal{D}(1,1)=D(1,1)$), and accumulated distances around the marginal areas of the matrix $\mathcal{D}$.  In the line 5-8, we use Eq.~\ref{eq:DTWdyproc} to compute $\mathcal{D}(i,j)$. The $d_{DTW}$ is reported at the line 9. In the line 10, we infer the optimal warping path by backtracking the steps from $\mathcal{D}(T_X,T_Y)$ to $\mathcal{D}(1,1)$ using the Algorithm~\ref{algo:DTWWarpingPathFunc}.

\setlength{\intextsep}{0pt}
\IncMargin{1em}
\begin{algorithm2e}
\caption{ WarpingPathFindingFunction}
\label{algo:DTWWarpingPathFunc}
\SetKwInOut{Input}{input}\SetKwInOut{Output}{output}
\Input{ $T_X\times T_Y$-Matrix $\mathcal{D}$.}
\Output{ DTW optimal warping path $P^*$. }
\begin{small}
\SetAlgoLined

\nl Set $P'(1) = (T_X,T_Y)$, $k=1$, and $\Delta^*=(T_X,T_Y)$\;
\nl \While{$\Delta^*\neq (1,1)$}
{
\nl    Let $(i,j)=P'(k)$ and $\mathcal{D}(a)= \mathcal{D}(a_1,a_2)$ where $a=(a_1,a_2)$ \;
\nl    Let $I\subseteq\{(i-1,j),(i-1,j-1),(i,j-1)\}$ s.t. $\forall (k,l) \in I, k\geq 1, l\geq 1$\;
\nl    $\Delta^*= \argmin_{\Delta \in I}\mathcal{D}(\Delta)$\;
\nl    $P'(k+1)=\Delta^*$\;
\nl    $k=k+1$\;
}
\nl Let $P'$ have a length $K$\;
\nl Let $P^*$ be the optimal warping path with length $K$ where $\forall i \in \{1,\dots,K\}, P^*(i)=P'(K-i+1)$\;
\nl Return $P^*$\;
\end{small}
\end{algorithm2e}\DecMargin{1em}

In Algorithm~\ref{algo:DTWWarpingPathFunc}, starting at the cell $\mathcal{D}(T_X,T_Y)$ (line 1), we search for the neighbor cell in $\mathcal{D}$ that have the lowest accumulative distance ($\Delta^*= \argmin_{\Delta \in I}\mathcal{D}(\Delta)$). Then, we mark the minimum-distance neighbor cell ($P'(k+1)=\Delta^*$)  as well as jumping to the marked cell ($k=k+1$) and continue for the next iteration (line 2-6). We repeat the steps of marking the minimum-distance neighbor cell until we meet the $\mathcal{D}(1,1)$ cell. The list of all marked cells is the optimal warping path ($P^*$). 

\begin{figure}
\centering
\includegraphics[width=.95\columnwidth]{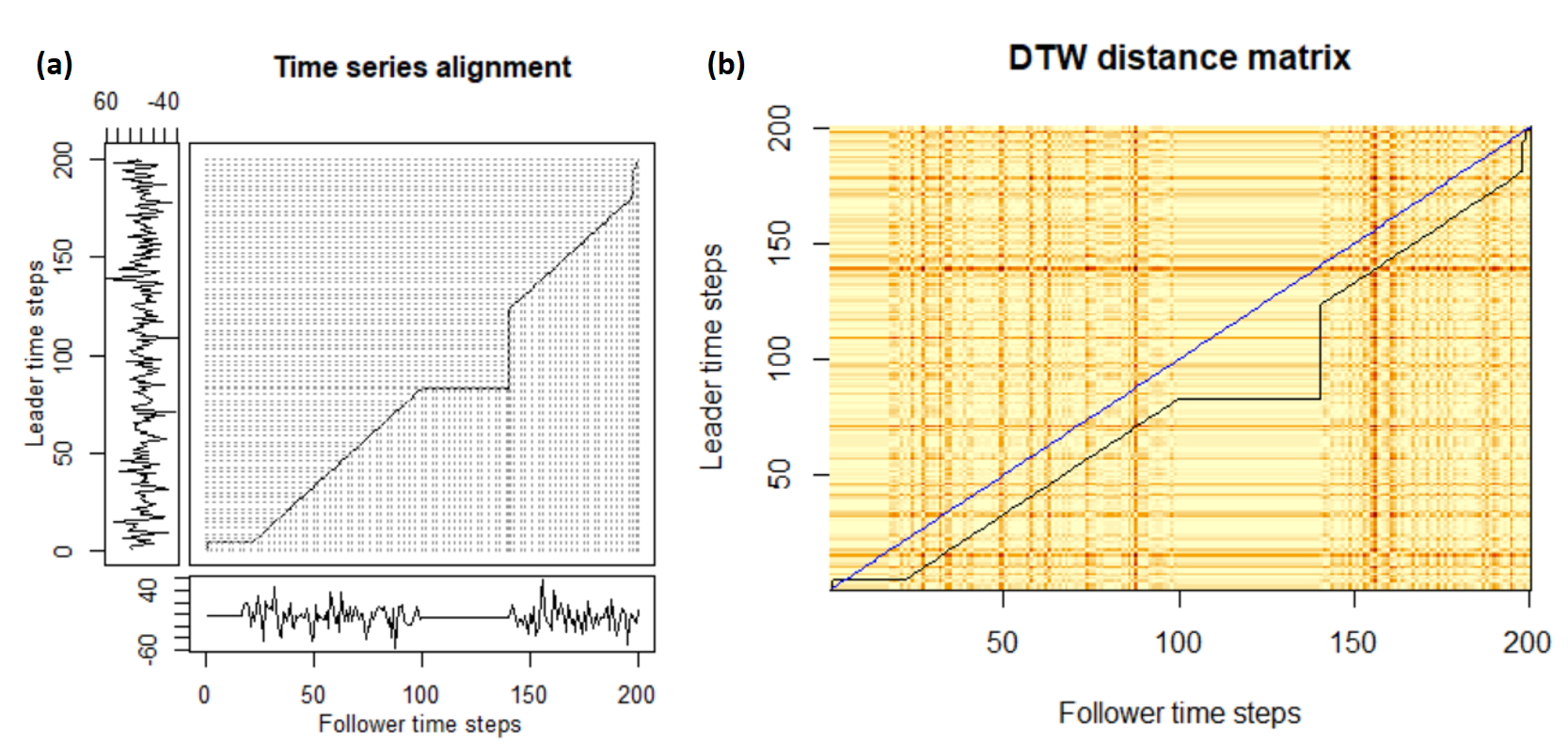}
\caption{ The example of time series alignment by DTW.  In this example, the follower time series imitates the leader with time delay $17$ time steps. Between 110th and 150th time steps, the follower constantly imitates leader at the 83th time step. (a) The matching between elements of follower time series and elements of leader time series. The black line is the optimal warping path. (b) The heatmap of accumulative distance matrix $\mathcal{D}$ of DTW and the optimal warping path (black line) compared against the diagonal line (blue line). A darker color represents a higher distance.   }
\label{fig:DTWex}
\end{figure}

Figrue~\ref{fig:DTWex} illustrates the example of DTW matching between two time series. In this example, the follower time series imitates the leader with time delay $17$ time steps. Then between 110th and 150th time steps, the follower constantly imitates leader at the 83th time step. The Figure~\ref{fig:DTWex} (a) shows the matching of elements between time series. The black line is the optimal warping path. We can see that, in the optimal warping path, the elements between 110th and 150th time steps of follower time series matched with the element of leader at the 83th time step. The Figure~\ref{fig:DTWex} (b) shows the DTW accumulative distance matrix $\mathcal{D}$.  The optimal warping path is in the black color, while the blue line is a diagonal line of the matrix. A darker color represents a higher distance. We can see that the optimal warping path is below the diagonal line. This implies that the follower elements are matched with the leader elements back in time. Specifically, for any pair of indices $(i,j)$ within the optimal warping path, $\Delta= j-i>0$, when the optimal warping path is below the diagonal line. The element $Y(j)$ is matched with $X(i)$ in the past. Hence, we can infer whether $X\prec Y$ using their optimal warping path. 

\section{Appendix: VLTimeCausality package}
\label{apdx:manual}
The VLTimeCausality package contains the implementation of VL-Granger causality, Granger causality, and VL-Transfer entropy.  The package is available on the the Comprehensive R Archive Network (CRAN). This implies all R programming users can install our package anywhere. To install the package, we can use the following commands.

\begin{lstlisting}[language=R]
R>install.packages("VLTimeCausality")
\end{lstlisting}

To use the package, the first step is to use the provided function to generate simulation time series.
\begin{lstlisting}[language=R]
R>library(VLTimeCausality)
R>TS <- VLTimeCausality::SimpleSimulationVLtimeseries() 
\end{lstlisting}
The TS variable contains TS\$X and TS\$Y where TS\$X causes TS\$Y. Then, we can run VL-Granger causality with the $\gamma = 0.5$ below.
\begin{lstlisting}[language=R]
R>out<-VLTimeCausality::VLGrangerFunc(Y=TS$Y,X=TS$X, gamma= 0.5) 
\end{lstlisting}
The result of inference is below.
\begin{lstlisting}[language=R]
R> out$XgCsY
[1] TRUE
R> out$BICDiffRatio
[1] 0.8434518
\end{lstlisting}
It implies that TS\$X causes TS\$Y  (out\$XgCsY is true) with the BIC difference ratio at 0.84.

For the VL-Transfer Entropy, the following command is used to check whether TS\$X causes TS\$Y with the number of bootstrap replicates is 100 and the significance level $\alpha = 0.05$.
\begin{lstlisting}[language=R]
R> out2<-VLTransferEntropy(Y=TS$Y,X=TS$X,VLflag=TRUE,nboot=100, alpha = 0.05)
\end{lstlisting}
The result of inference is below.
\begin{lstlisting}[language=R]
R> out2$XgCsY_trns
[1] TRUE
R> out2$TEratio
[1] 4.539785
R> out2$pval
[1] 0
\end{lstlisting}
It implies that TS\$X causes TS\$Y  (out2\$XgCsY\_trns is true) with the  transfer entropy ratio at 4.54 and the p-value is at 0. For more details about functions and parameters in the packages, please see  \href{https://cran.r-project.org/package=VLTimeCausality}{https://cran.r-project.org/package=VLTimeCausality}.

\section{Appendix: Simulation generating code}
\label{apdx:SimCode}
The following code was used to generate simulation datasets that were analyzed and reported the results in Section~\ref{sec:simResPairwise}. We deployed the ``rmatio'' package~\cite{rmatio} for files operation handling. Although our simulation datasets were generated randomly, we set the random seeds to make it being able to be replicated. 

\begin{lstlisting}[language=R]
library(rmatio)
library(VLTimeCausality)
origSeed<-2020
set.seed(origSeed)
rounds<-15
seeds<-runif(rounds,1000,250000)
simType1DataSets<-list()

# normal gen
for(i in seq(rounds))
{
  simType1DataSets[["normalPos"]][[i]]<- SimpleSimulationVLtimeseries(expflag = FALSE, arimaFlag = FALSE,causalFlag = TRUE, seedVal = seeds[i] )
  simType1DataSets[["normalNeg"]][[i]]<- SimpleSimulationVLtimeseries(expflag = FALSE, arimaFlag = FALSE,causalFlag = FALSE, seedVal = seeds[i] )
  
  simType1DataSets[["ARMAPos"]][[i]]<- SimpleSimulationVLtimeseries(expflag = FALSE, arimaFlag = TRUE,causalFlag = TRUE, seedVal = seeds[i] )
  simType1DataSets[["ARMANeg"]][[i]]<- SimpleSimulationVLtimeseries(expflag = FALSE, arimaFlag = TRUE,causalFlag = FALSE, seedVal = seeds[i] )
  
  simType1DataSets[["normalARMANeg"]][[i]]<- simType1DataSets[["normalNeg"]][[i]]
  simType1DataSets[["normalARMANeg"]][[i]]$Y<-simType1DataSets[["ARMANeg"]][[i]]$Y
  
}
simType1DataSets$origSeed<-origSeed
simType1DataSets$seeds<-seeds
simType1DataSets$rounds<-rounds
write.mat(file = "simType1DataSets.mat",object = simType1DataSets)
\end{lstlisting}

The following code was used to generate simulation datasets that were analyzed and reported the results in Section~\ref{sec:simResGroup}. 

\begin{lstlisting}[language=R]
library(rmatio)
library(VLTimeCausality)
origSeed<-2020
set.seed(origSeed)
rounds<-15
seeds<-runif(rounds,1000,250000)
simType2DataSets<-list()

# normal gen
for(i in seq(rounds))
{
  
  TS<- MultipleSimulationVLtimeseries(seedVal = seeds[i], arimaFlag = FALSE)
  simType2DataSets[["normal"]][[i]]<-TS
  simType2DataSets[["normalX"]][[i]]<-rowMeans(TS[,1:3])
  simType2DataSets[["normalY"]][[i]]<-rowMeans(TS[,4:10])

  
  TS<- MultipleSimulationVLtimeseries(seedVal = seeds[i], arimaFlag = TRUE)
  simType2DataSets[["ARMA"]][[i]]<-TS
  simType2DataSets[["ARMAX"]][[i]]<-rowMeans(TS[,1:3])
  simType2DataSets[["ARMAY"]][[i]]<-rowMeans(TS[,4:10])

}
simType2DataSets$origSeed<-origSeed
simType2DataSets$seeds<-seeds
simType2DataSets$rounds<-rounds
write.mat(file = "simType2DataSets.mat",object = simType2DataSets)
\end{lstlisting}

\balance
\bibliographystyle{ACM-Reference-Format}

\begin{thebibliography}{52}


\ifx \showCODEN    \undefined \def \showCODEN     #1{\unskip}     \fi
\ifx \showDOI      \undefined \def \showDOI       #1{#1}\fi
\ifx \showISBNx    \undefined \def \showISBNx     #1{\unskip}     \fi
\ifx \showISBNxiii \undefined \def \showISBNxiii  #1{\unskip}     \fi
\ifx \showISSN     \undefined \def \showISSN      #1{\unskip}     \fi
\ifx \showLCCN     \undefined \def \showLCCN      #1{\unskip}     \fi
\ifx \shownote     \undefined \def \shownote      #1{#1}          \fi
\ifx \showarticletitle \undefined \def \showarticletitle #1{#1}   \fi
\ifx \showURL      \undefined \def \showURL       {\relax}        \fi
\providecommand\bibfield[2]{#2}
\providecommand\bibinfo[2]{#2}
\providecommand\natexlab[1]{#1}
\providecommand\showeprint[2][]{arXiv:#2}

\bibitem[\protect\citeauthoryear{??}{MLS}{[n. d.]}]%
        {MLSourcecode}
 \bibinfo{year}{[n. d.]}\natexlab{}.
\newblock \bibinfo{title}{Granger Causality Package in MATLAB}.
\newblock
  \bibinfo{howpublished}{\url{https://www.mathworks.com/matlabcentral/fileexchange/25467-granger-causality-test}}.
\newblock


\bibitem[\protect\citeauthoryear{??}{RSo}{[n. d.]}]%
        {RSourcecode}
 \bibinfo{year}{[n. d.]}\natexlab{}.
\newblock \bibinfo{title}{Granger Causality Package in R}.
\newblock
  \bibinfo{howpublished}{\url{https://www.rdocumentation.org/packages/MSBVAR/versions/0.9-2/topics/granger.test}}.
\newblock


\bibitem[\protect\citeauthoryear{Amornbunchornvej}{Amornbunchornvej}{[n. d.]}]%
        {ShareSourcecode}
\bibfield{author}{\bibinfo{person}{Chainarong Amornbunchornvej}.}
  \bibinfo{year}{[n. d.]}\natexlab{}.
\newblock \bibinfo{title}{VLTimeSeriesCausality: R package for variable-lag
  causal inference in time series}.
\newblock
  \bibinfo{howpublished}{\url{https://github.com/DarkEyes/VLTimeSeriesCausality}}.
\newblock
\newblock
\shownote{Accessed: 2019-12-10.}


\bibitem[\protect\citeauthoryear{Amornbunchornvej, Brugere, Strandburg-Peshkin,
  Farine, Crofoot, and Berger-Wolf}{Amornbunchornvej et~al\mbox{.}}{2018}]%
        {FLICAtkdd}
\bibfield{author}{\bibinfo{person}{Chainarong Amornbunchornvej},
  \bibinfo{person}{Ivan Brugere}, \bibinfo{person}{Ariana Strandburg-Peshkin},
  \bibinfo{person}{Damien Farine}, \bibinfo{person}{Margaret~C Crofoot}, {and}
  \bibinfo{person}{Tanya~Y Berger-Wolf}.} \bibinfo{year}{2018}\natexlab{}.
\newblock \showarticletitle{Coordination Event Detection and Initiator
  Identification in Time Series Data}.
\newblock \bibinfo{journal}{\emph{ACM Trans. Knowl. Discov. Data}}
  \bibinfo{volume}{12}, \bibinfo{number}{5}, Article \bibinfo{articleno}{53}
  (\bibinfo{date}{6} \bibinfo{year}{2018}), \bibinfo{numpages}{33}~pages.
\newblock
\showISSN{1556-4681}
\urldef\tempurl%
\url{https://doi.org/10.1145/3201406}
\showDOI{\tempurl}


\bibitem[\protect\citeauthoryear{Amornbunchornvej, Zheleva, and
  Berger-Wolf}{Amornbunchornvej et~al\mbox{.}}{2019}]%
        {VLGranger}
\bibfield{author}{\bibinfo{person}{Chainarong Amornbunchornvej},
  \bibinfo{person}{Elena Zheleva}, {and} \bibinfo{person}{Tanya Berger-Wolf}.}
  \bibinfo{year}{2019}\natexlab{}.
\newblock \showarticletitle{Variable-lag Granger Causality for Time Series
  Analysis}. In \bibinfo{booktitle}{\emph{2019 IEEE International Conference on
  Data Science and Advanced Analytics (DSAA)}}. IEEE, \bibinfo{pages}{21--30}.
\newblock
\urldef\tempurl%
\url{https://doi.org/10.1109/DSAA.2019.00016}
\showDOI{\tempurl}


\bibitem[\protect\citeauthoryear{Arnold, Liu, and Abe}{Arnold
  et~al\mbox{.}}{2007}]%
        {Arnold:2007:TCM:1281192.1281203}
\bibfield{author}{\bibinfo{person}{Andrew Arnold}, \bibinfo{person}{Yan Liu},
  {and} \bibinfo{person}{Naoki Abe}.} \bibinfo{year}{2007}\natexlab{}.
\newblock \showarticletitle{Temporal Causal Modeling with Graphical Granger
  Methods}. In \bibinfo{booktitle}{\emph{Proceedings of the 13th ACM SIGKDD
  International Conference on Knowledge Discovery and Data Mining}}
  \emph{(\bibinfo{series}{KDD '07})}. \bibinfo{publisher}{ACM},
  \bibinfo{address}{New York, NY, USA}, \bibinfo{pages}{66--75}.
\newblock
\showISBNx{978-1-59593-609-7}
\urldef\tempurl%
\url{https://doi.org/10.1145/1281192.1281203}
\showDOI{\tempurl}


\bibitem[\protect\citeauthoryear{Atukeren et~al\mbox{.}}{Atukeren
  et~al\mbox{.}}{2010}]%
        {atukeren2010relationship}
\bibfield{author}{\bibinfo{person}{Erdal Atukeren} {et~al\mbox{.}}}
  \bibinfo{year}{2010}\natexlab{}.
\newblock \showarticletitle{The relationship between the F-test and the Schwarz
  criterion: implications for Granger-causality tests}.
\newblock \bibinfo{journal}{\emph{Econ Bull}} \bibinfo{volume}{30},
  \bibinfo{number}{1} (\bibinfo{year}{2010}), \bibinfo{pages}{494--499}.
\newblock


\bibitem[\protect\citeauthoryear{Azzalini and Bowman}{Azzalini and
  Bowman}{1990}]%
        {azzalini1990look}
\bibfield{author}{\bibinfo{person}{Adelchi Azzalini} {and}
  \bibinfo{person}{Adrian~W Bowman}.} \bibinfo{year}{1990}\natexlab{}.
\newblock \showarticletitle{A look at some data on the Old Faithful geyser}.
\newblock \bibinfo{journal}{\emph{Journal of the Royal Statistical Society:
  Series C (Applied Statistics)}} \bibinfo{volume}{39}, \bibinfo{number}{3}
  (\bibinfo{year}{1990}), \bibinfo{pages}{357--365}.
\newblock


\bibitem[\protect\citeauthoryear{Barnett, Barrett, and Seth}{Barnett
  et~al\mbox{.}}{2009}]%
        {PhysRevLett.103.238701}
\bibfield{author}{\bibinfo{person}{Lionel Barnett}, \bibinfo{person}{Adam~B.
  Barrett}, {and} \bibinfo{person}{Anil~K. Seth}.}
  \bibinfo{year}{2009}\natexlab{}.
\newblock \showarticletitle{Granger Causality and Transfer Entropy Are
  Equivalent for Gaussian Variables}.
\newblock \bibinfo{journal}{\emph{Phys. Rev. Lett.}}  \bibinfo{volume}{103}
  (\bibinfo{date}{Dec} \bibinfo{year}{2009}), \bibinfo{pages}{238701}.
\newblock
Issue 23.
\urldef\tempurl%
\url{https://doi.org/10.1103/PhysRevLett.103.238701}
\showDOI{\tempurl}


\bibitem[\protect\citeauthoryear{Behrendt, Dimpfl, Peter, and
  Zimmermann}{Behrendt et~al\mbox{.}}{2019}]%
        {BEHRENDT2019100265}
\bibfield{author}{\bibinfo{person}{Simon Behrendt}, \bibinfo{person}{Thomas
  Dimpfl}, \bibinfo{person}{Franziska~J. Peter}, {and}
  \bibinfo{person}{David~J. Zimmermann}.} \bibinfo{year}{2019}\natexlab{}.
\newblock \showarticletitle{RTransferEntropy — Quantifying information flow
  between different time series using effective transfer entropy}.
\newblock \bibinfo{journal}{\emph{SoftwareX}}  \bibinfo{volume}{10}
  (\bibinfo{year}{2019}), \bibinfo{pages}{100265}.
\newblock
\showISSN{2352-7110}
\urldef\tempurl%
\url{https://doi.org/10.1016/j.softx.2019.100265}
\showDOI{\tempurl}


\bibitem[\protect\citeauthoryear{Box, Jenkins, Reinsel, and Ljung}{Box
  et~al\mbox{.}}{2015}]%
        {box2015time}
\bibfield{author}{\bibinfo{person}{George~EP Box}, \bibinfo{person}{Gwilym~M
  Jenkins}, \bibinfo{person}{Gregory~C Reinsel}, {and} \bibinfo{person}{Greta~M
  Ljung}.} \bibinfo{year}{2015}\natexlab{}.
\newblock \bibinfo{booktitle}{\emph{Time series analysis: forecasting and
  control}}.
\newblock \bibinfo{publisher}{John Wiley \& Sons}.
\newblock


\bibitem[\protect\citeauthoryear{Chazelle}{Chazelle}{2011}]%
        {doi:10.1137/100791671}
\bibfield{author}{\bibinfo{person}{Bernard Chazelle}.}
  \bibinfo{year}{2011}\natexlab{}.
\newblock \showarticletitle{The Total s-Energy of a Multiagent System}.
\newblock \bibinfo{journal}{\emph{SIAM Journal on Control and Optimization}}
  \bibinfo{volume}{49}, \bibinfo{number}{4} (\bibinfo{year}{2011}),
  \bibinfo{pages}{1680--1706}.
\newblock
\urldef\tempurl%
\url{https://doi.org/10.1137/100791671}
\showDOI{\tempurl}
\showeprint{https://doi.org/10.1137/100791671}


\bibitem[\protect\citeauthoryear{Chen, Rangarajan, Feng, and Ding}{Chen
  et~al\mbox{.}}{2004}]%
        {chen2004analyzing}
\bibfield{author}{\bibinfo{person}{Yonghong Chen}, \bibinfo{person}{Govindan
  Rangarajan}, \bibinfo{person}{Jianfeng Feng}, {and} \bibinfo{person}{Mingzhou
  Ding}.} \bibinfo{year}{2004}\natexlab{}.
\newblock \showarticletitle{Analyzing multiple nonlinear time series with
  extended Granger causality}.
\newblock \bibinfo{journal}{\emph{Physics Letters A}} \bibinfo{volume}{324},
  \bibinfo{number}{1} (\bibinfo{year}{2004}), \bibinfo{pages}{26--35}.
\newblock


\bibitem[\protect\citeauthoryear{Crofoot, Kays, and Wikelski}{Crofoot
  et~al\mbox{.}}{2015}]%
        {crofoot2015data}
\bibfield{author}{\bibinfo{person}{Margaret~C Crofoot},
  \bibinfo{person}{Roland~W Kays}, {and} \bibinfo{person}{Martin Wikelski}.}
  \bibinfo{year}{2015}\natexlab{}.
\newblock \bibinfo{title}{Data from: Shared decision-making drives collective
  movement in wild baboons}.
\newblock
\newblock


\bibitem[\protect\citeauthoryear{Dimpfl and Peter}{Dimpfl and Peter}{2013}]%
        {dimpfl2013using}
\bibfield{author}{\bibinfo{person}{Thomas Dimpfl} {and}
  \bibinfo{person}{Franziska~Julia Peter}.} \bibinfo{year}{2013}\natexlab{}.
\newblock \showarticletitle{Using transfer entropy to measure information flows
  between financial markets}.
\newblock \bibinfo{journal}{\emph{Studies in Nonlinear Dynamics \&
  Econometrics}} \bibinfo{volume}{17}, \bibinfo{number}{1}
  (\bibinfo{year}{2013}), \bibinfo{pages}{85--102}.
\newblock


\bibitem[\protect\citeauthoryear{Eichler}{Eichler}{2013}]%
        {doi:10.1098/rsta.2011.0613}
\bibfield{author}{\bibinfo{person}{Michael Eichler}.}
  \bibinfo{year}{2013}\natexlab{}.
\newblock \showarticletitle{Causal inference with multiple time series:
  principles and problems}.
\newblock \bibinfo{journal}{\emph{Philosophical Transactions of the Royal
  Society A: Mathematical, Physical and Engineering Sciences}}
  \bibinfo{volume}{371}, \bibinfo{number}{1997} (\bibinfo{year}{2013}),
  \bibinfo{pages}{20110613}.
\newblock
\urldef\tempurl%
\url{https://doi.org/10.1098/rsta.2011.0613}
\showDOI{\tempurl}


\bibitem[\protect\citeauthoryear{Giorgino et~al\mbox{.}}{Giorgino
  et~al\mbox{.}}{2009}]%
        {giorgino2009computing}
\bibfield{author}{\bibinfo{person}{Toni Giorgino} {et~al\mbox{.}}}
  \bibinfo{year}{2009}\natexlab{}.
\newblock \showarticletitle{Computing and visualizing dynamic time warping
  alignments in R: the dtw package}.
\newblock \bibinfo{journal}{\emph{Journal of statistical Software}}
  \bibinfo{volume}{31}, \bibinfo{number}{7} (\bibinfo{year}{2009}),
  \bibinfo{pages}{1--24}.
\newblock


\bibitem[\protect\citeauthoryear{Granger and Jeon}{Granger and Jeon}{2004}]%
        {granger2004forecasting}
\bibfield{author}{\bibinfo{person}{Clive Granger} {and} \bibinfo{person}{Yongil
  Jeon}.} \bibinfo{year}{2004}\natexlab{}.
\newblock \showarticletitle{Forecasting performance of information criteria
  with many macro series}.
\newblock \bibinfo{journal}{\emph{Journal of Applied Statistics}}
  \bibinfo{volume}{31}, \bibinfo{number}{10} (\bibinfo{year}{2004}),
  \bibinfo{pages}{1227--1240}.
\newblock


\bibitem[\protect\citeauthoryear{Granger}{Granger}{1969}]%
        {granger1969investigating}
\bibfield{author}{\bibinfo{person}{Clive~WJ Granger}.}
  \bibinfo{year}{1969}\natexlab{}.
\newblock \showarticletitle{Investigating causal relations by econometric
  models and cross-spectral methods}.
\newblock \bibinfo{journal}{\emph{Econometrica: Journal of the Econometric
  Society}} (\bibinfo{year}{1969}), \bibinfo{pages}{424--438}.
\newblock


\bibitem[\protect\citeauthoryear{Griveau-Billion and
  Calderhead}{Griveau-Billion and Calderhead}{2019}]%
        {griveau2019efficient}
\bibfield{author}{\bibinfo{person}{Th{\'e}ophile Griveau-Billion} {and}
  \bibinfo{person}{Ben Calderhead}.} \bibinfo{year}{2019}\natexlab{}.
\newblock \showarticletitle{Efficient structure learning with automatic
  sparsity selection for causal graph processes}.
\newblock \bibinfo{journal}{\emph{arXiv preprint arXiv:1906.04479}}
  (\bibinfo{year}{2019}).
\newblock


\bibitem[\protect\citeauthoryear{Iseki, Mukuta, Ushiki, and Harada}{Iseki
  et~al\mbox{.}}{2019}]%
        {iseki-aaai19}
\bibfield{author}{\bibinfo{person}{Akane Iseki}, \bibinfo{person}{Y. Mukuta},
  \bibinfo{person}{Y. Ushiki}, {and} \bibinfo{person}{T. Harada}.}
  \bibinfo{year}{2019}\natexlab{}.
\newblock \showarticletitle{Estimating the causal effect from partially
  observed time series}. In \bibinfo{booktitle}{\emph{AAAI}}.
\newblock


\bibitem[\protect\citeauthoryear{Janzing and Scholkopf}{Janzing and
  Scholkopf}{2010}]%
        {janzing2010causal}
\bibfield{author}{\bibinfo{person}{Dominik Janzing} {and}
  \bibinfo{person}{Bernhard Scholkopf}.} \bibinfo{year}{2010}\natexlab{}.
\newblock \showarticletitle{Causal inference using the algorithmic Markov
  condition}.
\newblock \bibinfo{journal}{\emph{IEEE Transactions on Information Theory}}
  \bibinfo{volume}{56}, \bibinfo{number}{10} (\bibinfo{year}{2010}),
  \bibinfo{pages}{5168--5194}.
\newblock


\bibitem[\protect\citeauthoryear{Keogh and Pazzani}{Keogh and Pazzani}{2001}]%
        {keogh2001derivative}
\bibfield{author}{\bibinfo{person}{Eamonn~J Keogh} {and}
  \bibinfo{person}{Michael~J Pazzani}.} \bibinfo{year}{2001}\natexlab{}.
\newblock \showarticletitle{Derivative dynamic time warping}. In
  \bibinfo{booktitle}{\emph{Proceedings of the 2001 SIAM international
  conference on data mining}}. SIAM, \bibinfo{pages}{1--11}.
\newblock


\bibitem[\protect\citeauthoryear{Kontoyiannis and Skoularidou}{Kontoyiannis and
  Skoularidou}{2016}]%
        {kontoyiannis2016estimating}
\bibfield{author}{\bibinfo{person}{Ioannis Kontoyiannis} {and}
  \bibinfo{person}{Maria Skoularidou}.} \bibinfo{year}{2016}\natexlab{}.
\newblock \showarticletitle{Estimating the directed information and testing for
  causality}.
\newblock \bibinfo{journal}{\emph{IEEE Transactions on Information Theory}}
  \bibinfo{volume}{62}, \bibinfo{number}{11} (\bibinfo{year}{2016}),
  \bibinfo{pages}{6053--6067}.
\newblock


\bibitem[\protect\citeauthoryear{Lee, Nemati, Silva, Edwards, Butler, and
  Malhotra}{Lee et~al\mbox{.}}{2012}]%
        {lee2012transfer}
\bibfield{author}{\bibinfo{person}{Joon Lee}, \bibinfo{person}{Shamim Nemati},
  \bibinfo{person}{Ikaro Silva}, \bibinfo{person}{Bradley~A Edwards},
  \bibinfo{person}{James~P Butler}, {and} \bibinfo{person}{Atul Malhotra}.}
  \bibinfo{year}{2012}\natexlab{}.
\newblock \showarticletitle{Transfer entropy estimation and directional
  coupling change detection in biomedical time series}.
\newblock \bibinfo{journal}{\emph{Biomedical engineering online}}
  \bibinfo{volume}{11}, \bibinfo{number}{1} (\bibinfo{year}{2012}),
  \bibinfo{pages}{19}.
\newblock


\bibitem[\protect\citeauthoryear{Liu, Bahadori, and Li}{Liu
  et~al\mbox{.}}{2012}]%
        {liu2012sparse}
\bibfield{author}{\bibinfo{person}{Yan Liu}, \bibinfo{person}{Taha Bahadori},
  {and} \bibinfo{person}{Hongfei Li}.} \bibinfo{year}{2012}\natexlab{}.
\newblock \showarticletitle{Sparse-gev: Sparse latent space model for
  multivariate extreme value time serie modeling}. In
  \bibinfo{booktitle}{\emph{ICML}}.
\newblock


\bibitem[\protect\citeauthoryear{Malinsky and Spirtes}{Malinsky and
  Spirtes}{2018}]%
        {malinsky2018causal}
\bibfield{author}{\bibinfo{person}{Daniel Malinsky} {and}
  \bibinfo{person}{Peter Spirtes}.} \bibinfo{year}{2018}\natexlab{}.
\newblock \showarticletitle{Causal structure learning from multivariate time
  series in settings with unmeasured confounding}. In
  \bibinfo{booktitle}{\emph{Proceedings of 2018 ACM SIGKDD Workshop on Causal
  Discovery}}. \bibinfo{pages}{23--47}.
\newblock


\bibitem[\protect\citeauthoryear{Mueen and Keogh}{Mueen and Keogh}{2016}]%
        {Mueen:2016:EOP:2939672.2945383}
\bibfield{author}{\bibinfo{person}{Abdullah Mueen} {and}
  \bibinfo{person}{Eamonn Keogh}.} \bibinfo{year}{2016}\natexlab{}.
\newblock \showarticletitle{Extracting Optimal Performance from Dynamic Time
  Warping}. In \bibinfo{booktitle}{\emph{Proceedings of the 22Nd ACM SIGKDD
  International Conference on Knowledge Discovery and Data Mining}}
  \emph{(\bibinfo{series}{KDD '16})}. \bibinfo{publisher}{ACM},
  \bibinfo{address}{New York, NY, USA}, \bibinfo{pages}{2129--2130}.
\newblock
\showISBNx{978-1-4503-4232-2}
\urldef\tempurl%
\url{https://doi.org/10.1145/2939672.2945383}
\showDOI{\tempurl}


\bibitem[\protect\citeauthoryear{Pearl}{Pearl}{2000}]%
        {pearl2000causality}
\bibfield{author}{\bibinfo{person}{J Pearl}.} \bibinfo{year}{2000}\natexlab{}.
\newblock \showarticletitle{Causality: Models, reasoning and inference
  Cambridge University Press}.
\newblock \bibinfo{journal}{\emph{Cambridge, MA, USA,}}  \bibinfo{volume}{9}
  (\bibinfo{year}{2000}).
\newblock


\bibitem[\protect\citeauthoryear{Peng, Sun, Rose, and Li}{Peng
  et~al\mbox{.}}{2007}]%
        {Peng:2007:SSI:1288552.1288557}
\bibfield{author}{\bibinfo{person}{Wei Peng}, \bibinfo{person}{Tong Sun},
  \bibinfo{person}{Philip Rose}, {and} \bibinfo{person}{Tao Li}.}
  \bibinfo{year}{2007}\natexlab{}.
\newblock \showarticletitle{A Semi-automatic System with an Iterative Learning
  Method for Discovering the Leading Indicators in Business Processes}. In
  \bibinfo{booktitle}{\emph{Proceedings of the 2007 International Workshop on
  Domain Driven Data Mining}} \emph{(\bibinfo{series}{DDDM '07})}.
  \bibinfo{publisher}{ACM}, \bibinfo{address}{New York, NY, USA},
  \bibinfo{pages}{33--42}.
\newblock
\showISBNx{978-1-59593-846-6}
\urldef\tempurl%
\url{https://doi.org/10.1145/1288552.1288557}
\showDOI{\tempurl}


\bibitem[\protect\citeauthoryear{Peters, Janzing, and Sch{\"o}lkopf}{Peters
  et~al\mbox{.}}{2013}]%
        {peters2013causal}
\bibfield{author}{\bibinfo{person}{Jonas Peters}, \bibinfo{person}{Dominik
  Janzing}, {and} \bibinfo{person}{Bernhard Sch{\"o}lkopf}.}
  \bibinfo{year}{2013}\natexlab{}.
\newblock \showarticletitle{Causal inference on time series using restricted
  structural equation models}. In \bibinfo{booktitle}{\emph{Advances in Neural
  Information Processing Systems}}. \bibinfo{pages}{154--162}.
\newblock


\bibitem[\protect\citeauthoryear{Peters, Janzing, and Sch{\"o}lkopf}{Peters
  et~al\mbox{.}}{2017}]%
        {peters2017elements}
\bibfield{author}{\bibinfo{person}{Jonas Peters}, \bibinfo{person}{Dominik
  Janzing}, {and} \bibinfo{person}{Bernhard Sch{\"o}lkopf}.}
  \bibinfo{year}{2017}\natexlab{}.
\newblock \bibinfo{booktitle}{\emph{Elements of causal inference: foundations
  and learning algorithms}}.
\newblock \bibinfo{publisher}{MIT press}.
\newblock


\bibitem[\protect\citeauthoryear{Quinn, Kiyavash, and Coleman}{Quinn
  et~al\mbox{.}}{2015}]%
        {7273888}
\bibfield{author}{\bibinfo{person}{C.~J. Quinn}, \bibinfo{person}{N. Kiyavash},
  {and} \bibinfo{person}{T.~P. Coleman}.} \bibinfo{year}{2015}\natexlab{}.
\newblock \showarticletitle{Directed Information Graphs}.
\newblock \bibinfo{journal}{\emph{IEEE Transactions on Information Theory}}
  \bibinfo{volume}{61}, \bibinfo{number}{12} (\bibinfo{date}{Dec}
  \bibinfo{year}{2015}), \bibinfo{pages}{6887--6909}.
\newblock
\showISSN{0018-9448}
\urldef\tempurl%
\url{https://doi.org/10.1109/TIT.2015.2478440}
\showDOI{\tempurl}


\bibitem[\protect\citeauthoryear{Raffalovich, Deane, Armstrong, and
  Tsao}{Raffalovich et~al\mbox{.}}{2008}]%
        {raffalovich2008model}
\bibfield{author}{\bibinfo{person}{Lawrence~E Raffalovich},
  \bibinfo{person}{Glenn~D Deane}, \bibinfo{person}{David Armstrong}, {and}
  \bibinfo{person}{Hui-Shien Tsao}.} \bibinfo{year}{2008}\natexlab{}.
\newblock \showarticletitle{Model selection procedures in social research:
  Monte-Carlo simulation results}.
\newblock \bibinfo{journal}{\emph{Journal of Applied Statistics}}
  \bibinfo{volume}{35}, \bibinfo{number}{10} (\bibinfo{year}{2008}),
  \bibinfo{pages}{1093--1114}.
\newblock


\bibitem[\protect\citeauthoryear{Sakoe and Chiba}{Sakoe and Chiba}{1978}]%
        {sakoe1978dynamic}
\bibfield{author}{\bibinfo{person}{Hiroaki Sakoe} {and} \bibinfo{person}{Seibi
  Chiba}.} \bibinfo{year}{1978}\natexlab{}.
\newblock \showarticletitle{Dynamic programming algorithm optimization for
  spoken word recognition}.
\newblock \bibinfo{journal}{\emph{IEEE transactions on acoustics, speech, and
  signal processing}} \bibinfo{volume}{26}, \bibinfo{number}{1}
  (\bibinfo{year}{1978}), \bibinfo{pages}{43--49}.
\newblock


\bibitem[\protect\citeauthoryear{Sch{\"o}lkopf, Janzing, Peters, Sgouritsa,
  Zhang, and Mooij}{Sch{\"o}lkopf et~al\mbox{.}}{2012}]%
        {scholkopf2012causal}
\bibfield{author}{\bibinfo{person}{Bernhard Sch{\"o}lkopf},
  \bibinfo{person}{Dominik Janzing}, \bibinfo{person}{Jonas Peters},
  \bibinfo{person}{Eleni Sgouritsa}, \bibinfo{person}{Kun Zhang}, {and}
  \bibinfo{person}{Joris Mooij}.} \bibinfo{year}{2012}\natexlab{}.
\newblock \showarticletitle{On causal and anticausal learning}. In
  \bibinfo{booktitle}{\emph{ICML}}.
\newblock


\bibitem[\protect\citeauthoryear{Schreiber}{Schreiber}{2000}]%
        {schreiber-prl00}
\bibfield{author}{\bibinfo{person}{Thomas Schreiber}.}
  \bibinfo{year}{2000}\natexlab{}.
\newblock \showarticletitle{Measuring information transfer}.
\newblock \bibinfo{journal}{\emph{Physical review letters}}
  \bibinfo{volume}{85}, \bibinfo{number}{2} (\bibinfo{year}{2000}),
  \bibinfo{pages}{461}.
\newblock


\bibitem[\protect\citeauthoryear{Schwab, Miladinovic, and Karlen}{Schwab
  et~al\mbox{.}}{2019}]%
        {schwab-aaai19}
\bibfield{author}{\bibinfo{person}{Patrick Schwab}, \bibinfo{person}{Djordje
  Miladinovic}, {and} \bibinfo{person}{Walter Karlen}.}
  \bibinfo{year}{2019}\natexlab{}.
\newblock \showarticletitle{Granger-causal attentive Mixtures of Experts:
  Learning Important Features with Neural Networks}. In
  \bibinfo{booktitle}{\emph{AAAI}}.
\newblock


\bibitem[\protect\citeauthoryear{Senin}{Senin}{2008}]%
        {senin2008dynamic}
\bibfield{author}{\bibinfo{person}{Pavel Senin}.}
  \bibinfo{year}{2008}\natexlab{}.
\newblock \showarticletitle{Dynamic time warping algorithm review}.
\newblock \bibinfo{journal}{\emph{Information and Computer Science Department
  University of Hawaii at Manoa Honolulu, USA}} \bibinfo{volume}{855},
  \bibinfo{number}{1-23} (\bibinfo{year}{2008}), \bibinfo{pages}{40}.
\newblock


\bibitem[\protect\citeauthoryear{Shajarisales, Janzing, Sch{\"o}lkopf, and
  Besserve}{Shajarisales et~al\mbox{.}}{2015}]%
        {shajarisales2015telling}
\bibfield{author}{\bibinfo{person}{Naji Shajarisales}, \bibinfo{person}{Dominik
  Janzing}, \bibinfo{person}{Bernhard Sch{\"o}lkopf}, {and}
  \bibinfo{person}{Michel Besserve}.} \bibinfo{year}{2015}\natexlab{}.
\newblock \showarticletitle{Telling cause from effect in deterministic linear
  dynamical systems}. In \bibinfo{booktitle}{\emph{ICML}}.
  \bibinfo{pages}{285--294}.
\newblock


\bibitem[\protect\citeauthoryear{{Shannon}}{{Shannon}}{1948}]%
        {6773024}
\bibfield{author}{\bibinfo{person}{C.~E. {Shannon}}.}
  \bibinfo{year}{1948}\natexlab{}.
\newblock \showarticletitle{A mathematical theory of communication}.
\newblock \bibinfo{journal}{\emph{The Bell System Technical Journal}}
  \bibinfo{volume}{27}, \bibinfo{number}{3} (\bibinfo{date}{July}
  \bibinfo{year}{1948}), \bibinfo{pages}{379--423}.
\newblock
\showISSN{0005-8580}
\urldef\tempurl%
\url{https://doi.org/10.1002/j.1538-7305.1948.tb01338.x}
\showDOI{\tempurl}


\bibitem[\protect\citeauthoryear{Shao, Guo, Luk, and Weston}{Shao
  et~al\mbox{.}}{2014}]%
        {shao2014accelerating}
\bibfield{author}{\bibinfo{person}{Shengjia Shao}, \bibinfo{person}{Ce Guo},
  \bibinfo{person}{Wayne Luk}, {and} \bibinfo{person}{Stephen Weston}.}
  \bibinfo{year}{2014}\natexlab{}.
\newblock \showarticletitle{Accelerating transfer entropy computation}. In
  \bibinfo{booktitle}{\emph{2014 International Conference on Field-Programmable
  Technology (FPT)}}. IEEE, \bibinfo{pages}{60--67}.
\newblock


\bibitem[\protect\citeauthoryear{Shibuya, Harada, and Kuniyoshi}{Shibuya
  et~al\mbox{.}}{2009}]%
        {shibuya-kdd09}
\bibfield{author}{\bibinfo{person}{Takashi Shibuya}, \bibinfo{person}{Tatsuya
  Harada}, {and} \bibinfo{person}{Yasuo Kuniyoshi}.}
  \bibinfo{year}{2009}\natexlab{}.
\newblock \showarticletitle{Causality quantification and its applications:
  structuring and modeling of multivariate time series}. In
  \bibinfo{booktitle}{\emph{KDD}}. \bibinfo{publisher}{ACM}.
\newblock


\bibitem[\protect\citeauthoryear{Siggiridou and Kugiumtzis}{Siggiridou and
  Kugiumtzis}{2016}]%
        {siggiridou2016granger}
\bibfield{author}{\bibinfo{person}{Elsa Siggiridou} {and}
  \bibinfo{person}{Dimitris Kugiumtzis}.} \bibinfo{year}{2016}\natexlab{}.
\newblock \showarticletitle{Granger causality in multivariate time series using
  a time-ordered restricted vector autoregressive model}.
\newblock \bibinfo{journal}{\emph{IEEE Transactions on Signal Processing}}
  \bibinfo{volume}{64}, \bibinfo{number}{7} (\bibinfo{year}{2016}),
  \bibinfo{pages}{1759--1773}.
\newblock


\bibitem[\protect\citeauthoryear{Sliva, Reilly, Casstevens, and
  Chamberlain}{Sliva et~al\mbox{.}}{2015}]%
        {sliva2015tools}
\bibfield{author}{\bibinfo{person}{Amy Sliva}, \bibinfo{person}{Scott~Neal
  Reilly}, \bibinfo{person}{Randy Casstevens}, {and} \bibinfo{person}{John
  Chamberlain}.} \bibinfo{year}{2015}\natexlab{}.
\newblock \showarticletitle{Tools for validating causal and predictive claims
  in social science models}.
\newblock \bibinfo{journal}{\emph{Procedia Manufacturing}}  \bibinfo{volume}{3}
  (\bibinfo{year}{2015}), \bibinfo{pages}{3925--3932}.
\newblock


\bibitem[\protect\citeauthoryear{Spirtes, Glymour, and Scheines}{Spirtes
  et~al\mbox{.}}{1993}]%
        {Spirtes1993}
\bibfield{author}{\bibinfo{person}{Peter Spirtes}, \bibinfo{person}{Clark
  Glymour}, {and} \bibinfo{person}{Richard Scheines}.}
  \bibinfo{year}{1993}\natexlab{}.
\newblock \bibinfo{booktitle}{\emph{Discovery Algorithms for Causally
  Sufficient Structures}}.
\newblock \bibinfo{publisher}{Springer New York}, \bibinfo{address}{New York,
  NY}, \bibinfo{pages}{103--162}.
\newblock
\showISBNx{978-1-4612-2748-9}
\urldef\tempurl%
\url{https://doi.org/10.1007/978-1-4612-2748-9_5}
\showDOI{\tempurl}


\bibitem[\protect\citeauthoryear{Strandburg-Peshkin and
  et~al.}{Strandburg-Peshkin and et~al.}{2013}]%
        {strandburg2013visual}
\bibfield{author}{\bibinfo{person}{A. Strandburg-Peshkin} {and}
  \bibinfo{person}{et al.}} \bibinfo{year}{2013}\natexlab{}.
\newblock \showarticletitle{Visual sensory networks and effective information
  transfer in animal groups}.
\newblock \bibinfo{journal}{\emph{Current Biology}} \bibinfo{volume}{23},
  \bibinfo{number}{17} (\bibinfo{year}{2013}), \bibinfo{pages}{R709--R711}.
\newblock


\bibitem[\protect\citeauthoryear{Strandburg-Peshkin, Farine, Couzin, and
  Crofoot}{Strandburg-Peshkin et~al\mbox{.}}{2015}]%
        {strandburg2015shared}
\bibfield{author}{\bibinfo{person}{Ariana Strandburg-Peshkin},
  \bibinfo{person}{Damien~R Farine}, \bibinfo{person}{Iain~D Couzin}, {and}
  \bibinfo{person}{Margaret~C Crofoot}.} \bibinfo{year}{2015}\natexlab{}.
\newblock \showarticletitle{Shared decision-making drives collective movement
  in wild baboons}.
\newblock \bibinfo{journal}{\emph{Science}} \bibinfo{volume}{348},
  \bibinfo{number}{6241} (\bibinfo{year}{2015}), \bibinfo{pages}{1358--1361}.
\newblock


\bibitem[\protect\citeauthoryear{Sun, Li, Liu, Chow, Sun, and Wang}{Sun
  et~al\mbox{.}}{2015}]%
        {Sun2015}
\bibfield{author}{\bibinfo{person}{Youqiang Sun}, \bibinfo{person}{Jiuyong Li},
  \bibinfo{person}{Jixue Liu}, \bibinfo{person}{Christopher Chow},
  \bibinfo{person}{Bingyu Sun}, {and} \bibinfo{person}{Rujing Wang}.}
  \bibinfo{year}{2015}\natexlab{}.
\newblock \showarticletitle{Using causal discovery for feature selection in
  multivariate numerical time series}.
\newblock \bibinfo{journal}{\emph{Machine Learning}} \bibinfo{volume}{101},
  \bibinfo{number}{1} (\bibinfo{date}{01 Oct} \bibinfo{year}{2015}),
  \bibinfo{pages}{377--395}.
\newblock
\showISSN{1573-0565}
\urldef\tempurl%
\url{https://doi.org/10.1007/s10994-014-5460-1}
\showDOI{\tempurl}


\bibitem[\protect\citeauthoryear{Varian}{Varian}{2016}]%
        {Varian7310}
\bibfield{author}{\bibinfo{person}{Hal~R. Varian}.}
  \bibinfo{year}{2016}\natexlab{}.
\newblock \showarticletitle{Causal inference in economics and marketing}.
\newblock \bibinfo{journal}{\emph{Proceedings of the National Academy of
  Sciences}} \bibinfo{volume}{113}, \bibinfo{number}{27}
  (\bibinfo{year}{2016}), \bibinfo{pages}{7310--7315}.
\newblock
\showISSN{0027-8424}
\urldef\tempurl%
\url{https://doi.org/10.1073/pnas.1510479113}
\showDOI{\tempurl}
\showeprint{https://www.pnas.org/content/113/27/7310.full.pdf}


\bibitem[\protect\citeauthoryear{Widgren and Hulbert}{Widgren and
  Hulbert}{2019}]%
        {rmatio}
\bibfield{author}{\bibinfo{person}{Stefan Widgren} {and}
  \bibinfo{person}{Christopher Hulbert}.} \bibinfo{year}{2019}\natexlab{}.
\newblock \bibinfo{booktitle}{\emph{rmatio: Read and Write 'Matlab' Files}}.
\newblock
\urldef\tempurl%
\url{https://CRAN.R-project.org/package=rmatio}
\showURL{%
\tempurl}
\newblock
\shownote{R package version 0.14.0.}


\bibitem[\protect\citeauthoryear{Yuan, Li, Zhang, and Qin}{Yuan
  et~al\mbox{.}}{2016}]%
        {yuan2016deep}
\bibfield{author}{\bibinfo{person}{Tao Yuan}, \bibinfo{person}{Gang Li},
  \bibinfo{person}{Zhaohui Zhang}, {and} \bibinfo{person}{S~Joe Qin}.}
  \bibinfo{year}{2016}\natexlab{}.
\newblock \showarticletitle{Deep causal mining for plant-wide oscillations with
  multilevel Granger causality analysis}. In \bibinfo{booktitle}{\emph{American
  Control Conference (ACC), 2016}}. IEEE, \bibinfo{pages}{5056--5061}.
\newblock


\end{thebibliography}

\balance

\end{document}